\documentclass{article}

\usepackage{style}
\usepackage[square,numbers]{natbib}

\usepackage{amsmath,amsfonts,bm}

\def\eqref#1{equation~\ref{#1}}

\def\1{\bm{1}}

\def\eps{{\epsilon}}

\DeclareMathAlphabet{\mathsfit}{\encodingdefault}{\sfdefault}{m}{sl}
\SetMathAlphabet{\mathsfit}{bold}{\encodingdefault}{\sfdefault}{bx}{n}

\usepackage[utf8]{inputenc} %
\usepackage[T1]{fontenc}    %
\usepackage{hyperref}       %
\usepackage{url}            %
\usepackage{booktabs}       %
\usepackage{amsfonts}       %
\usepackage{nicefrac}       %
\usepackage{microtype}      %
\usepackage{xcolor}    

\newcommand\blfootnote[1]{%
  \begingroup
  \renewcommand\thefootnote{}\footnote{#1}%
  \addtocounter{footnote}{-1}%
  \endgroup
}

\title{LLP-Bench: A Large Scale Tabular Benchmark for \\Learning from Label Proportions}

\author{%
Anand Brahmbhatt* \\ Google Research India \\  {\tt anandpareshb@google.com}
\and
Mohith Pokala* \\ Google Research India \\ {\tt mohithpokala@google.com}
\and
Rishi Saket \\ Google Research India \\ {\tt rishisaket@google.com}
\and
Aravindan Raghuveer \\ Google Research India \\ {\tt araghuveer@google.com}
}

\usepackage[utf8]{inputenc}
\usepackage[english]{babel}
\usepackage{algorithm2e}
\SetKwComment{Comment}{/* }{ */}
\RestyleAlgo{ruled}

\usepackage{xcolor}
\usepackage{mathtools,amsfonts,amssymb,mdframed,xspace,xfrac,multicol,bbm,bm}
\usepackage{makecell}
\usepackage{eucal}
\usepackage{amssymb,amsmath,amsopn,mathtools}
\usepackage{mathrsfs}
\usepackage{enumitem}
\usepackage{subcaption}
\usepackage{dsfont}
\usepackage{wrapfig}
\usepackage{framed}
\usepackage{longtable}
\usepackage{placeins}
\usepackage{supertabular}
\usepackage{graphicx}

\usepackage[page]{appendix} %

\usepackage{caption}
\usepackage{subcaption}
\renewcommand{\R}{\mathbb{R}}

\newcommand{\mc}[1]{\ensuremath{\mathcal{#1}}\xspace}
\newcommand{\mb}[1]{\ensuremath{\mathbf{#1}}\xspace}
\newcommand{\tn}[1]{\ensuremath{\textnormal{#1}}\xspace}

\newtheorem{theorem}{Theorem}[section]
\newtheorem{lemma}[theorem]{Lemma}
\newtheorem{definition}[theorem]{Definition}

\newtheorem{corollary}[theorem]{Corollary}

\begin{document}
\maketitle
\blfootnote{* -- equal contribution}
\begin{abstract}

In the task of Learning from Label Proportions (LLP), a model is trained on groups (a.k.a bags) of instances and their corresponding label proportions to predict labels for individual instances. LLP has been applied pre-dominantly on two types of datasets - image and tabular. In image LLP, bags of fixed size are created by randomly sampling instances from an underlying dataset. Bags created via this methodology are called {\em random bags}. Experimentation on Image LLP has been mostly on random bags on CIFAR-* and MNIST datasets.
Despite being a very crucial task in privacy sensitive applications, tabular LLP does not yet have a open, large scale LLP benchmark.  One of the unique properties of tabular LLP is the ability to create {\em feature bags} where all the instances in a bag have the same value for a given feature. It has been shown in prior research that feature bags are very common in practical, real world applications~\cite{chen2023learning,SRR}. 

In this paper, we address the lack of a open, large scale tabular benchmark. 
First we propose LLP-Bench, a suite of 70 LLP datasets (62 feature bag and 8 random bag datasets) created from the Criteo CTR prediction and the Criteo Sponsored Search Conversion Logs datasets, the former a classification and the latter a regression dataset. These LLP datasets represent diverse ways in which bags can be constructed from underlying tabular data. To the best of our knowledge, LLP-Bench is the first large scale tabular LLP benchmark with an extensive diversity in constituent datasets. Second, we propose four metrics that characterize and quantify the hardness of a LLP dataset. Using these four metrics we present deep analysis of the 62 feature bag datasets in LLP-Bench. Finally we present the performance of 9 SOTA and popular tabular LLP techniques on all the 62 datasets.

\end{abstract}

\section{Introduction}
In traditional supervised learning, \emph{training} data consists of feature-vectors (instances) along with their labels.  A  model trained using such data is then used during inference to predict the labels of \emph{test} instances. 
In recent times, primarily due to privacy concerns and relative rarity of high quality large-scale supervised data, the weakly supervised framework of \emph{learning from label proportions} (LLP) has gained importance~\cite{MutCon,SRR,Ruping10,WIBB} In LLP, the training data is aggregated into \textit{bags}. Each bag contains a bunch of instances (and their feature vectors) and their corresponding aggregated label count.  
The goal is to learn a classification model for predicting the class-labels of individual instances~\cite{FreitasK05,Musicant}.

 Study of LLP has recently gained importance due to developments in the privacy landscape. In particular, restrictions on tracking of user events have led to an LLP formulation of user-modeling in \emph{online advertising}~\cite{DBLP:journals/corr/abs-2201-12666}. Since only the average label for a bag of users is revealed, the size of the bags is a measure of the privacy afforded. Other applications include medical records anonymization~\cite{WIBB},  IVF prediction~\cite{hernandez2018}, image classification~\cite{Bortsova18, Orting16}, mass spectrometry~\cite{chen2004cost}, datasets  with legal constraints~\cite{Ruping10,WIBB} and inadequate or costly supervision~\cite{DNRS,chen2004cost}.

Such LLP techniques have primarily been evaluated and studied on 
image \cite{liu2019learning, zhang2022learning, tsai2020learning, liu2022self, dulac2019deep, bortsova2018deep}  and tabular \cite{SRR, easy-llp, QuadriantoSCL09} datasets. On images, well known datasets like CIFAR-10, CIFAR-100, MNIST are used -- typically by randomly partitioning the dataset into bags -- to create medium-large scale LLP datasets. %
On the other hand, tabular data consists of independent rows of feature vectors with one more labels attached to each feature vector. Often, previous works used tabular LLP datasets derived from small UCI ~\cite{Dua:2019} datasets which  fail to simulate the diversity and scale of applications involving such data. 
Notably, tabular datasets are extremely common in real world classification and regression tasks for online advertising \cite{DBLP:journals/corr/abs-2201-12666}, health care research ~\cite{Ruping10,WIBB} and scientific simulation studies \cite{chen2004cost}. Such applications tend to use very large scale data either from online user interaction~\cite{mcmahan2013ad,he2014practical} or user studies~\cite{florencio2007large}. Impact of privacy leaks due to inadvertent exposure of sensitive data is much higher in large scale datasets.  Therefore LLP on large scale tabular datasets is a very critical application that is receiving increasing attention from the research community.  While image LLP has large scale benchmark datasets derived from CIFAR-*, an equivalent benchmark  does not exist for tabular data.

\begin{figure}[t!]
\centering
\vspace{-7mm}
\includegraphics[scale=0.17] {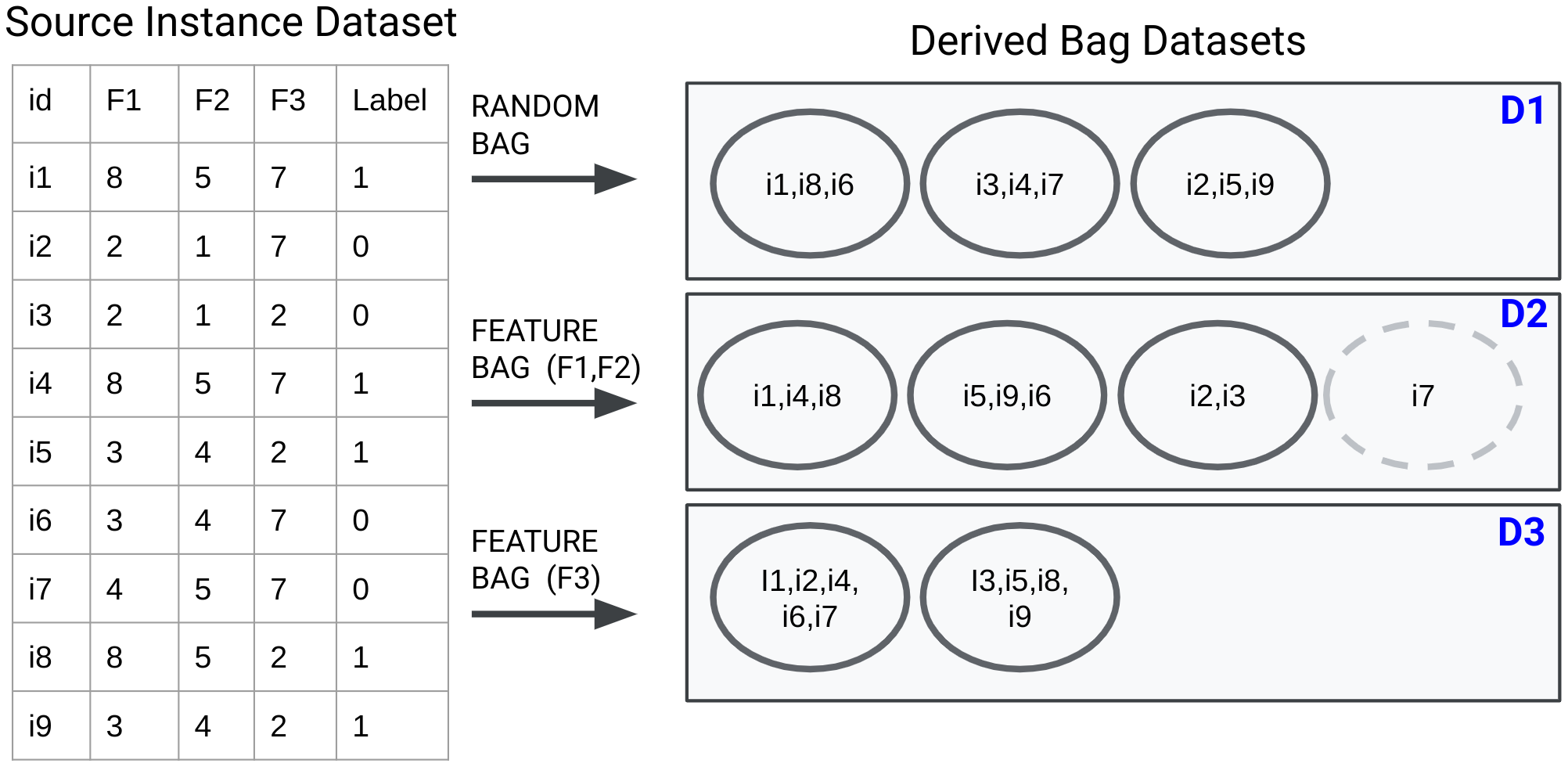}%
\caption{Dataset D1 is formed by randomly choosing without replacement from the instance dataset to form bags of size =3. Dataset D2 is feature bag formed by creating bags such that within a bag instances have the same value for features F1, F2. The fourth bag is removed because it has only one instance i7. Dataset D3 is similarly formed by using feature F3 as the grouping key. Notice the bags have become substantially larger than those in D1.}
\label{fig:dataset_schematic}
\end{figure} 

A unique property of tabular LLP as pointed out in  recent literature \cite{SRR,chen2023learning} is the notion of feature bags. In feature bags, bags are constructed such that all instances within the bag have the same value for given key(s) called the grouping key(s). Such bags occur in critical real-world applications such as user modelling in online advertising where the conversion labels are aggregated over pre-selected categorical features \cite{chen2023learning, easy-llp, DBLP:journals/corr/abs-2201-12666}. Figure~\ref{fig:dataset_schematic} shows three datasets created from the same instance datasets.   The first dataset is created like random sampling much like the ones created using CIFAR-*~\cite{liu2019learning, zhang2022learning, tsai2020learning}. The second and third datasets are features bags created by  using F1, F2  and F3 as the grouping keys respectively. Note that feature bags can also be made to have fixed size. 

Motivated by the above observations and the richness of the problem, we propose LLPBench a large scale, diverse benchmark for tabular LLP. We make three contributions in this paper as listed below.

\noindent
{\bf 1.} We propose LLP-Bench, a suite of 
\begin{itemize}
    \item 56 LLP datasets (52 feature bag and 4 random bag datasets) created from the popular Criteo CTR prediction dataset consisting of 45 million instances~\cite{Criteo:2014}.
    \item 14 LLP datasets (10 feature bag and 4 random bag datasets) created from the more recently available Criteo Sponsored Search Conversion Logs (Criteo SSCL) containing **TBA** million instances~\cite{tallis2018reacting}. 
\end{itemize}
These LLP datasets represent diverse ways in which bags can be constructed from underlying tabular data. Those feature bag datasets entail removing bags of extreme sizes, and contain 13.5 million to 24.75 million bags (**this should be instances, change numbers also**).  To the best of our knowledge, LLP-Bench is the first large scale tabular LLP benchmark with an extensive diversity in constituent datasets (Section~\ref{sec:dataset_creation}).

\noindent
{\bf 2.} We propose four metrics that help  characterize and quantify the hardness of an LLP dataset. Using these four` metrics we present deep analysis of the 56 datasets in LLP-Bench (Sections~\ref{sec:prelims},~\ref{sec:benchmark-diversity}).

\noindent
{\bf 3.} We present the performance of 9 SOTA and popular tabular LLP techniques on all the 56 datasets of LLP-Bench derived from the Criteo CTR prediction dataset. Some of these techniques are not applicable to regression datasets, so we evaluate 3 of them on the 14 datasets of LLP-Bench corresponding to Criteo Sponsored Search Conversion Logs.  To the best of our knowledge, our study consisting of more than 3000 experiments is the most extensive evaluation of popular tabular LLP techniques  in literature (Sections~\ref{sec:llpbench-performance},~\ref{sec:analysis}).

\medskip
{\bf Choice of base datasets.} The Criteo CTR and SSCL datasets are relevant to our work - firstly they are impression-click and click-conversion datasets which correspond to natural applications where LLP aggregation can occur due to user privacy~\cite{DBLP:journals/corr/abs-2201-12666}. Another reason is that they are among the few publicly available large scale tabular dataset with several categorical features, which allow for a rich collection of feature bag datasets. Our techniques for feature-based grouping and  LLP dataset analysis are more generally applicable. However, since we focus on a tabular LLP Benchmark, we restrict ourselves to these datasets in which either the feature-vectors represent impressions and are given binary $\{0,1\}$-labels indicating a click, or they represent clicks and the label indicates the number of conversions for that click. 

The binary label setting is widely studied in supervised machine learning, and in the LLP setting as well with many real-world applications: see for e.g. references \cite{hernandez2018} and \cite{DNRS} in the paper for applications in IVF prediction and high-energy physics. There are, on the other hand, regression applications in remote sensing~\cite{WL07,WRHOV08}  in which the real-valued labels are available only as aggregates.  In particular, the important task of user-modeling on online advertising platforms has recently seen privacy related restrictions leading to an LLP formulation for it (see Section 1 of \cite{easy-llp}), which is typically a classification or a regression problem.

\section{Related Work}\label{sec:related_work}

Several techniques for LLP have been studied over the years. The work of \cite{FreitasK05,HG} applied trained probabilistic models using Monte-Carlo methods. Subsequent works~\cite{Musicant,Ruping10} extended supervised learning techniques such neural nets, SVM and $k$-\emph{nearest neighbors} to LLP, others adapted clustering based approaches~\cite{ChenLQZ09,StolpeM11}, while \cite{YuLKJC13} proposed a novel $\propto$-SVM method for LLP. The work of \cite{QuadriantoSCL09} estimated model parameters from label proportions for the exponential generative model with certain assumptions on label distributions of bags. Their method was further applied by \cite{PatriniNCR14} for more general models and relaxed distributional assumptions. More recent works have investigated deep neural network based LLP methods~\cite{Bortsova18,ArdehalyC17,LiuWQTS19}, techniques using bag combinations~\cite{MutCon,SRR}, curated bags~\cite{CFKM23} and training on derived \emph{surrogate} labels for instances~\cite{easy-llp}. Recently, \cite{Saket21,Saket22} initiated a theoretical study of LLP from the computational learning perspective.

All of the previous works in LLP experimentally evaluate their methods on LLP datasets consisting of bags randomly created from some real-world supervised learning dataset. 
In these \emph{pseudo-synthetic}  LLP datasets, instances are randomly sampled/partitioned into the different bags, where in \cite{PatriniNCR14} and \cite{SRR} this process also clusters feature-vectors to generate more complicated bag distributions.
Almost all of the above works use limited scale data, typically small to medium scale UCI datasets~\cite{YuLKJC13,PatriniNCR14,MutCon}, image datasets~\cite{LiuWQTS19}, social media data~\cite{ArdehalyC17}. In general, there have been very few  large scale tabular datasets created and used for LLP. To the  best of our knowledge, \cite{SRR} and \cite{CFKM23} are the only works that explore a large dataset (Criteo) to test their methodology. Since the primary contribution  of these works is algorithmic, they do not justify their choice of how bags were created nor do they explore  the many  choices and tradeoffs of creating bags from a large scale instance dataset.  In our work we precisely address this gap  - we not only create benchmark with 56 diverse datasets, but we study in detail the tradeoffs involved and analyse the performance of 9 important baselines.    We study the performance of the following 9 baselies on the LLP-Bench benchmark.

${\bf DLLP}$: This a the standard LLP baseline method use in previous works~\cite{ArdehalyC17}, which optimizes a bag-level loss between the average label proportion and the predicted average label proportion. We evaluate ${\sf DLLP\text{-}BCE}$ and ${\sf DLLP\text{-}MSE}$ which use the bag-level BCE and MSE losses respectively using the minibatch training described above.\\ %
${\bf GenBags}$: This is the generalized bags method of \cite{SRR}, and since we only have one collection of bags in a given experiment we use random Gaussian combining weights to construct 120 different generalized bags per mini-batch of 8 bags. \\
${\bf Easy\text{-}LLP}$: In this technique proposed by \cite{easy-llp}, mean-bag labels along with the global average label are used to define a \emph{surrogate} label for training instances over which the model is optimized using the BCE loss.\\
${\bf OT}$ methods: There are the optimal transport based techniques proposed in \cite{OT,DZCBV19}. The first variant ${\bf OT\text{-}LLP}$ included is the non-regularized optimal transport for disjoint bags which can be implemented via a greedy approach. We also have ${\sf Hard\text{-}EROT\text{-}LLP}$ and ${\sf Soft\text{-}EROT\text{-}LLP}$ -- the hard and soft entropic regularized OT methods of \cite{OT}.  \\
${\bf SIM\text{-}LLP}$: In this method proposed by \cite{KotziasDFS15}, the bag-level DLLP loss is augmented with a pairwise similarity based loss penalizing different predictions of geometrically close feature-vectors. Since the similarity based loss has number of terms which is square in the number of feature-vectors in a minibatch, we sample a random set of $400$ feature-vectors from each minbatch to apply this loss. \\
${\bf Mean\text{-}Map}$: This is the well-known technique of \cite{QuadriantoSCL09} for linearized exponential generative models, consisting of two steps: computing the quantity using the bag-label proportions followed by optimizing for the model parameters over all feature-vectors. 
While the first step is a straightforward calculation, the second step is implemented using a minibatch optimization.

\medskip
While all of the above techniques are applicable to the binary-label setting and hence are evaluated on the LLP datasets derived from Criteo CTR, only a subset of them, specifically ${\sf DLLP\text{-}MSE}$, ${\bf GenBags}$ and ${\bf SIM\text{-}LLP}$, are applicable to real-valued labels and are included as baselines on the LLP datasets corresponding to Criteo SSCL.

\section{LLP Dataset Characteristics} \label{sec:prelims}

In this section, we first establish some notation, define LLP terminology and then propose four metrics to characterize and quantify the hardnesss of an LLP dataset. 

In our exploration of LLP we shall only consider non-negative real-valued instance labels.\\
\textbf{Notation}: $X := \{\bx^{(i)} \in \R^n\}_{i=1}^{m}$ is a dataset of $m$ feature vectors in $n$-dimensional space with labels given by $Y := \{y^{(i)} \in \{0, 1\}\}_{i=1}^{m}$. We denote by $\hat{Y} := \{\hat{y}^{(i)} \in \R_{\geq 0}\}_{i=1}^{m}$ the corresponding model predictions which are probabilities of the predicted label being $1$.
A \emph{bag} $B \subseteq [m]$ consists of feature vectors $X_B := \cup_{i \in B}\bx^{(i)}$ and with the corresponding label sum $y_B := \sum_{i \in B}y^{(i)}$. The label proportion of the bag is $y_B/|B|$. 
\begin{definition}[LLP  Dataset]
A \emph{learning from label proportions (LLP)} dataset corresponding to a collection of bags $\mathcal{B} := \{B_j\}_{j=1}^{N}$ is given by $\{(X_B, y_B)\,\mid\,B\in \mc{B}\}$. The \emph{label bias} of training dataset is $\mu(\mathcal{B}, Y) := \left(\sum_{B \in \mathcal{B}}y_B\right)/\left(\sum_{B \in \mathcal{B}}|B|\right)$, while the average label proportion is $\hat{\mu}(\mathcal{B}, Y) := \tfrac{1}{N}\left(\sum_{B \in \mathcal{B}}y_B/|B|\right)$. 
\end{definition}
The LLP datasets considered in the paper have disjoint bags. 

In the following we define statistics
comparing the separation among feature-vectors within bags and their separation across bags. 
using a natural notion of bag separation.
\begin{definition}[Bag Separation]\label{def:bagsep}
For a distance $d$ on $\R^n$ and collection of bags $\mc{B} = \{B_1, \dots, B_M\}$ the corresponding separation function is defined as ${\sf BagSep}(B, B', d) := \frac{1}{|B||B'|}\sum_{\bx \in B}\sum_{\bx' \in B'}d(\bx, \bx')$. We define the $M\times M$ matrix ${\sf BagSepMatrix}(\mc{B}, d)$ whose $(i,j)$th element is given by ${\sf BagSep}(B_i, B_j, d)$.
\end{definition}
We use ${\sf BagSep}$ to compute the average separation between pairs of bags and the average separation within each bag. If the feature-vectors in bags are clustered together and far away from those of other bags, we expect the former to be significantly greater than the later.
\begin{definition}[Inter-Bag Separation for a bag]
Given $\mathcal{B}$, and metric $d$ on $\R^n$, the average inter-bag distance for a bag $B \in \mathcal{B}$ is defined as ${\sf InterBagSep}(B, d) := \frac{1}{|\mathcal{B}|-1}\sum_{B' \in \mathcal{B}, B' \neq B}{\sf BagSep}(B, B', d)$. 
\end{definition}
For computing the average statistic for the entire dataset we define the following.
\begin{definition}
The mean intra-bag separation of $\mathcal{B}$ is defined as ${\sf MeanIntraBagSep}(\mathcal{B}, d) := \frac{1}{|\mathcal{B}|}\sum_{B \in \mathcal{B}}{\sf BagSep}(B, B, d)$. The mean of average inter-bag separation is defined as ${\sf MeanInterBagSep}(\mathcal{B}, d) :=  \frac{1}{|\mathcal{B}|}\sum_{B \in \mathcal{B}}{\sf InterBagSep}(B,  d)$.
\end{definition}

\subsection{Hardness metrics for LLP datasets}\label{sec:hardness_metrics}
We are now ready to present four metrics that characterize the hardness of a LLP dataset: (i) standard deviation of label proportion (ii) inter vs intra bag separation ratio (iii) mean bag size, and (iv) cumulative bag size distribution. 

\medskip
${\bf  LabelPropStdev}$: This  is the standard deviation of the label proportions of the collection of bags i.e., $\sqrt{\tn{Var}_{B\leftarrow \mc{B}}\left[y_B/|B|\right]}$. A higher ${\sf LabelPropStdev}$  typically provides more model training supervision. For e.g. consider $(X, Y)$ with two different bag collections $\mc{B}_1$ and $\mc{B}_2$ with $\alpha = \hat{\mu}(\mathcal{B}_1, Y) = \hat{\mu}(\mathcal{B}_1, Y)$  while $0 \approx \beta_1 = {\sf LabelPropStdev}(\mc{B}_1, Y) \ll {\sf LabelPropStdev}(\mc{B}_2, Y) =: \beta_2$. Consider a model that predicts $\alpha$ for every feature-vector. Since $\beta_1 \approx 0$, the predicted label proportions of this model are will be close to the true label proportions for most bags in $\mc{B}_1$ unlike in $\mc{B}_2$ (since $\beta_2 \gg \beta_2$). For many bag-level losses, this results in such a model being much closer to optimal for $\mc{B}_1$ rather $\mc{B}_2$. On the other hand, the model is only learning the average label proportion and not discriminating among the instances, therefore not desirable and is ruled out when ${\sf LabelPropStdev}$ is higher.

\medskip
${\bf InterIntraRatio}$: This denotes  ${\sf MeanInterBagSep}(\mc{B}, d)/{\sf MeanIntraBagSep}(\mc{B}, d)$ when $d = \ell_2^2$.
A dataset with large ${\sf InterIntraRatio}$ has well separated bags and therefore the label proportion supervision provided per bag carries more information and hence easier to learn from compared to a dataset with a smaller  ${\sf InterIntraRatio}$. 

\medskip
${\bf MeanBagSize}$: Since we have only have a label proportion for each bag, informally speaking, the larger the bag size the lower the amount of label supervision for that bag. The third and a simple metric to characterize this is ${\sf MeanBagSize}$ i.e., the mean size of all the bags in the dataset. Therefore, a dataset with larger mean bag size is a much harder dataset to learn from compared to one with a much smaller mean bag size.

\medskip
${\bf CumuBagSizeDist}$: 
The bag sizes for any dataset are characterized by their cumulative distribution function which plots the fraction of bags of size at most $t$ for all $t \geq 1$. 
We compute the bag sizes at the $50$, $70$, $85$ and $95$ percentile of cumulative distribution plot, for each dataset.
Short-tailed distributions have most bags of small size and a very few large sized bags whereas Long-tailed distributions contain many bags of large sizes. Bags of large sizes provide a very little label information for a lot of feature level information. Therefore a LLP dataset with a long-tailed distribution of bag sizes is a much harder dataset to learn from than a short-tailed one.

\section{LLP Dataset: Bag creation}
\label{sec:dataset_creation}

{\bf Criteo CTR.} This has 13 numerical and 26 categorical features and a binary label. Each of the approximately 45 million rows (instances) represents an impression (online ad) and the label indicates a click. The semantics of all the features are undisclosed and the values of all the categorical features hashed into 32-bits for anonymization. Additionally, the dataset has missing values. We use a preprocessed version of the dataset as done for the AutoInt~(\cite{autoint_cikm_2019}) model, described and implemented in their provided code\footnote{   {\scriptsize \tt https://github.com/DeepGraphLearning/RecommenderSystems/tree/master/featureRec} .} We choose AutoInt because that is among the best performing models on the Criteo benchmark\footnote{ {\scriptsize \tt https://paperswithcode.com/sota/click-through-rate-prediction-on-criteo}}.  For convenience we label the numerical and categorical features (in their order of occurrence) as $\tn{N}1, \dots, \tn{N}13$ and $\tn{C}1, \dots \tn{C}26$. The preprocessing applies $\tn{int}(\log^2(x))$ transformation when $x > 2$ on the numerical feature values $x$, and we further additively scale so that their values are non-negative integers. The categorical features are encoded as non-negative integers.

{\bf Criteo SSCL.} This has 3 numerical features ($\tilde{\tn{N}}1, \tilde{\tn{N}}2, \tilde{\tn{N}}3$) and 17 categorical features ($\tilde{\tn{C}}1,\dots, \tilde{\tn{C}}17$) and $1.7$ million instances in total. Unlike Criteo CTR, the columns in Criteo SSCL are labeled, and we provide the mapping in Appendix \ref{app:criteo_sscl_column_encoding}. The missing values of the dataset are handled by (i) removing the datapoints in which the target is missing, (ii) group all the infrequent categorical values (occurring at most 5 times) along with the missing value (which is -1) into one category for each categorical feature, and (iii) replace all missing values by the mean of that feature for each numerical column. Since this is a regression dataset we do not transform the numerical features, while the categorical features are encoded as non-negative integers. The label is a non-negative real value.

For both Criteo CTR and SSCL, we create two types of LLP Data sets. We create 4 {\em Random Bag} datasets by randomly sampling without replacement (Similar to D1 in Figure~\ref{fig:dataset_schematic}) to create bags of fixed sizes of 64, 128, 256, 512. Let $\mc{U}$ denote the categorical columns $\{\tn{C}1, \dots, \tn{C}26\}$ in the case of Criteo CTR and $\{\tilde{\tn{C}}1,\dots, \tilde{\tn{C}}17\}$ for Criteo SSCL.
Next, we create {\em Feature Bag} datasets by 
 grouping  instances by subsets $\mc{C} \subseteq \mc{U}$, of the categorical columns, where $\mc{C} \leq 2$. The feature grouping subset used for a dataset is called its grouping key.  For each setting of the values of $\mc{C}$ we obtain a bag with instances with those values of $\mc{C}$  (Similar to D2 and D3 in Figure~\ref{fig:dataset_schematic}).    Each such  key grouping yields an LLP dataset\footnote{ {\scriptsize Note that for model training purposes such bags may be created from only the \emph{train set} portion of the entire dataset}}.
Thus, we obtain ${26 \choose 2} + 26 = 351$ LLP datasets from Criteo CTR and ${17 \choose 2} + 17 = 153$ LLP datasets corresponding to Criteo SSCL, each referred to also as a \textit{dataset} on $\mc{C}$  ($|\mc{C}| \leq 2$). Note that for any dataset, the set of bags partition the dataset and therefore each instance occurs in exactly one bag. 

Next we describe two filtering strategies to remove datasets that are ineffective in practice. The first strategy removes very small or very large bags within a dataset. The second strategy drops a dataset entirely if the first filtering method results in pruning a large portion of the underlying dataset.

\subsection{Bag Filtering} \label{sec:bag_filtering}
The feature bag creation step leads to bags of varying sizes. Very small bags are not practical because they do not preserve enough privacy.
For instance, datasets on $\{\tn{C}10, \tn{C}16\}$ and $\{\tn{C}4, \tn{C}10\}$ from Criteo CTR each contain more than $8\times 10^6$ bags.
We introduce a hyper-parameter $low_{thresh}$ which represents the size of the smallest bag   that can be present in a LLP-Bench dataset.

Similarly, a very large bag is almost useless since the information lost via aggregation is significant and hence the dataset cannot be used to build a useful classifier.
For  instance, the initial dataset on $\tn{C}9$ creates only $3$ bags and the dataset on $\tn{C}20$ creates only $4$ bags. Similarly, for the Criteo SSCL Dataset, the initial dataset on $\tilde{\tn{C}}1$ created only $11$ bags and the dataset on $\tilde{\tn{C}}2$ creates only $4$ bags.
We introduce a hyper-parameter $high_{thresh}$. which represents the size of the largest bag  that can be present in a LLP-Bench dataset.

For our experiments we set $low_{thresh}$ to be 50 and $high_{thresh}$ to be 2500.

\subsection{Dataset Filtering}\label{sec:dataset_filtering}
\label{filter_sec}

If the $low_{thresh}$ and $high_{thresh}$ based filtering remove  a significant fraction of bags then most of the underlying instances will be lost. This will cause the LLP dataset's performance to be poor because we are not left with enough signal to train on. We hence drop datasets that have less than $instance_{thresh}\%$ of the original instance data size. 
We set $instance_{thresh}$ to be $30\%$ in our experiments.  
After applying this filter, we are left with $52$ LLP datasets from Criteo CTR and $10$ from Criteo SSCL. All the datasets in single columns are filtered out as the maximum percentage of instances any of these datasets retains is $21.68\%$ ($\tn{C}4$) for datasets formed using Criteo CTR dataset and $28.84\%$ ($\tilde{\tn{C}}5$) for datasets formed using Criteo SSCL dataset. This analysis is presented in detail in Appendix \ref{appendix:bag_stats}.
For notational convenience in the rest of this paper, we shall call $(A, B)$ where $A, B \in \mc{U}$ the LLP dataset formed via grouping by the subset $\mc{C} = \{A, B\}$.

In Appendix~\ref{appendix:feature_random_train} we demonstrate the creation and performance metrics of \emph{fixed size feature bag} datasets as well -- in which  for each of the 52 groupings from Criteo CTR and bag size $q \in \{64, 128, 256, 512\}$, the train instances are
 ordered according to the grouping features and consecutive $q$-sized sequences are made into bags.

\section{Diversity of the Benchmark} \label{sec:benchmark-diversity}

Figures \ref{fig:datasets_vs_bag_metrics} and  \ref{fig:datasets_performance_vs_bag_metrics_criteo_ssl} depict for each of the 52 feature bag LLP datsets from Crireo CTR and 10 from Criteo SSCL respectively chosen in Sec. \ref{sec:dataset_filtering}, the values of three different bag-level metrics: (i) ${\sf MeanBagSize}$ - the average size of bags, (ii) ${\sf LabelPropStdev}$ - the standard deviation of the bag label proportions, and (ii) ${\sf InterIntraRatio}$ as given defined in Sec. \ref{sec:hardness_metrics}. Apart from capturing the bag size and label proportion distribution, the third metric also quantifies the geometric distribution of the feature-vectors w.r.t the bags, in particular how clustered the feature-vectors in an average bag are.

{\bf Criteo CTR.} We see in Figure \ref{fig:mean-bag-size} that  ${\sf MeanBagSize}$ values range from nearly 500 to around 200. While most of them are in $[150,250]$, around one-fourths of the values are above $300$, indicating significant diversity in the values of this metric. 
In Figure \ref{fig:label-prop-stdev} we see a similar trend in ${\sf LabelPropStdev}$ -- while most of the values are in the range $[0.15,0.18]$, around 25\% of the them are below $0.14$. This unsurprising since larger bags would typically lead to more concentrated label proportions, and thus ${\sf MeanBagSize}$  is loosely anti-correlated to ${\sf LabelPropStdev}$ and our collection of datasets have similar diversity of the latter's values.
Figure \ref{fig:inter-intra-mean_ratio} has the values for ${\sf InterIntraRatio}$ showing that they are well spread across the range $[1.1, 1.6]$. While most values are below $1.4$, there are around 17\% of them which are above $1.5$, indicating that the distribution has a fat tail and metric values are diverse.

{\bf Criteo SSCL.} Similar trends are also observed from Figure \ref{fig:datasets_performance_vs_bag_metrics_criteo_ssl}, albeit on a smaller number of points. In Figure \ref{fig:mean-bag-size-criteo-ssl} ${\sf MeanBagSize}$ ranges from 200 to around 325, with values evenly distributed in the approximate range 200-240 except two values exceeding 280 indicating the diversity of this metric. The values of ${\sf LabelPropStdev}$ in Figure  \ref{fig:label-prop-stdev-criteo-ssl} has a diverse distribution in the range around 12.5 to 14 barring one outlier value near 11.5 corresponding to a dataset which has the minimum ${\sf MeanBagSize}$. Figure \ref{fig:inter-intra-mean_ratio-criteo-ssl} also indicates a fairly uniform spread of ${\sf InterIntraRatio}$ values in its range around 1.18 to around 1.26. 

\begin{figure}[htb]
    \centering
    \begin{minipage}{.48\linewidth}
    \begin{subfigure}[b]{\linewidth}
        \includegraphics[width=\linewidth]{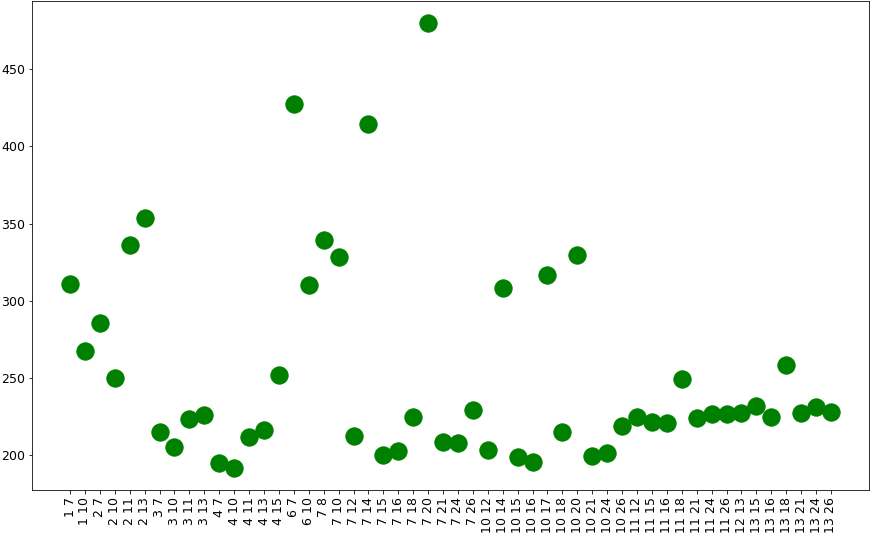}
        \caption{${\sf MeanBagSize}$}
        \label{fig:mean-bag-size}
    \end{subfigure}
    \begin{subfigure}[b]{\linewidth}
        \includegraphics[width=\linewidth]{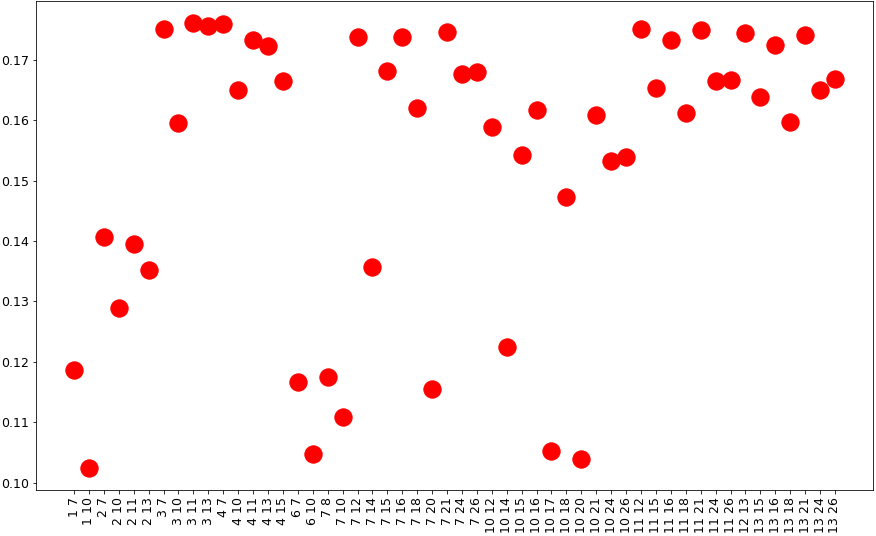}
        \caption{${\sf LabelPropStdev}$}
        \label{fig:label-prop-stdev}
    \end{subfigure}
    \begin{subfigure}[b]{\linewidth}
        \includegraphics[width=\linewidth]{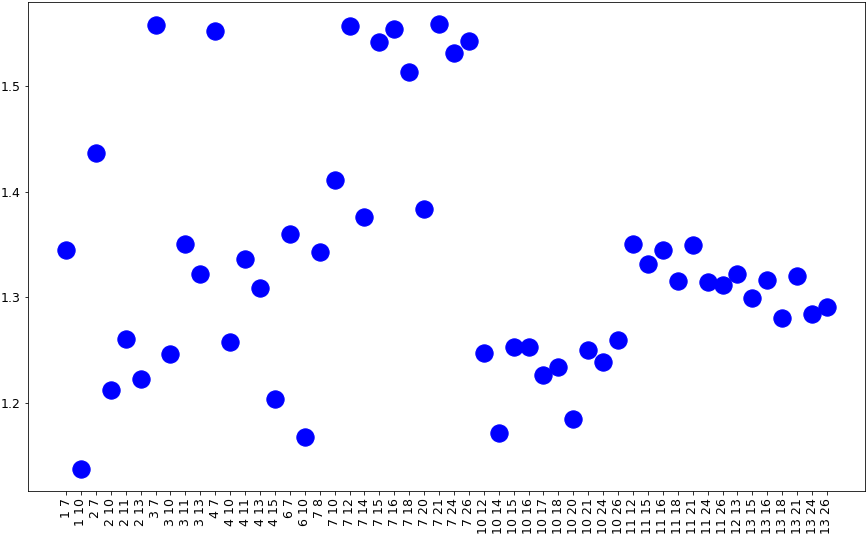}
        \caption{${\sf InterIntraRatio}$}
        \label{fig:inter-intra-mean_ratio}
    \end{subfigure}
    \caption{Datasets vs. bag-level metrics: $y$-axis has the metric, $x$-axis has the datsets.
    }
    \label{fig:datasets_vs_bag_metrics}
    \end{minipage}\quad
    \begin{minipage}{.48\linewidth}
    \begin{subfigure}[b]{\linewidth}
        \includegraphics[width=\linewidth]{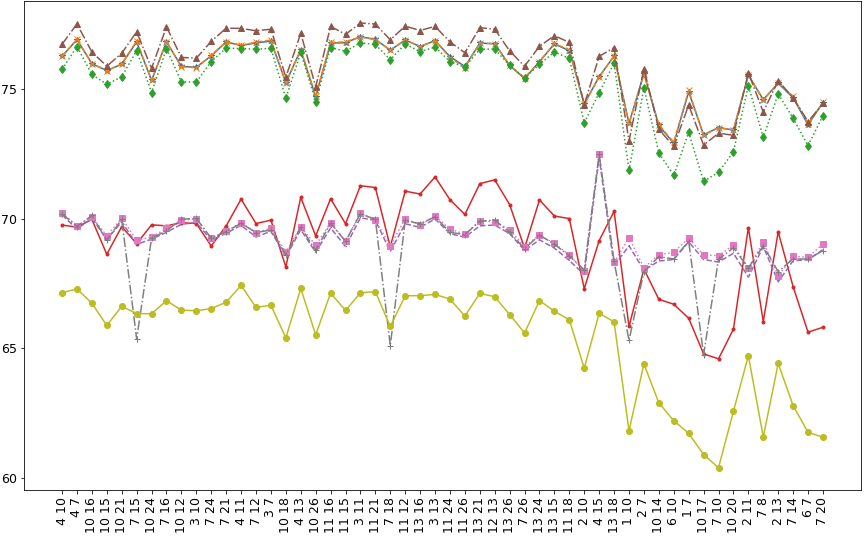}
        \caption{${\sf MeanBagSize}$}
        \label{fig:mean-bag-size-2}
    \end{subfigure}
    \begin{subfigure}[b]{\linewidth}
        \includegraphics[width=\linewidth]{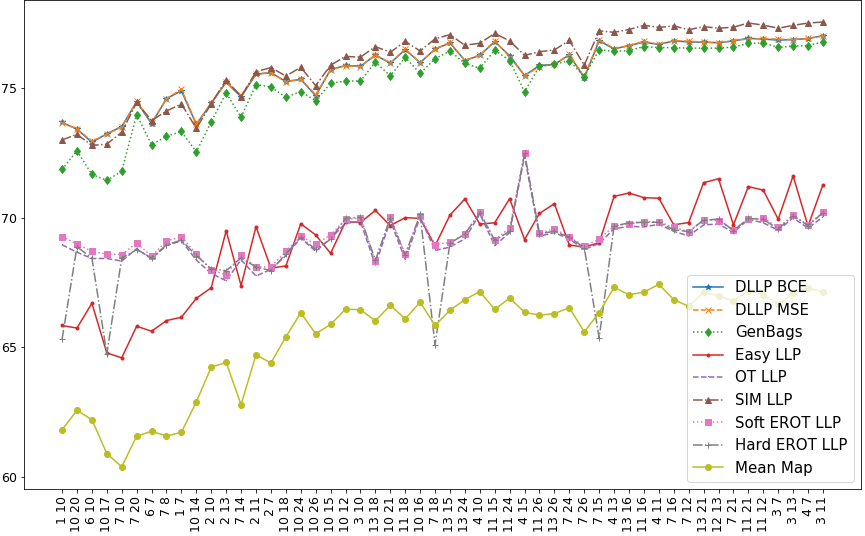}
        \caption{${\sf LabelPropStdev}$}
        \label{fig:label-prop-stdev-2}
    \end{subfigure}
    \begin{subfigure}[b]{\linewidth}
        \includegraphics[width=\linewidth]{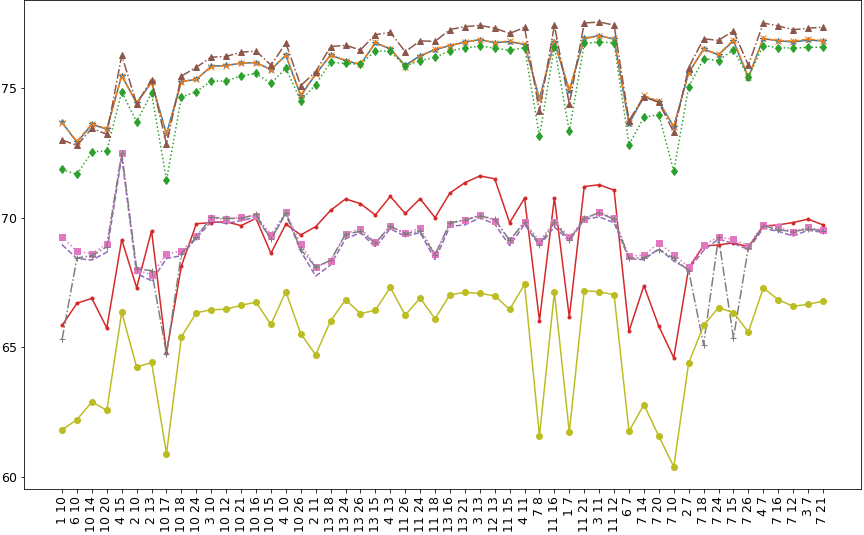}
        \caption{${\sf InterIntraRatio}$}
        \label{fig:inter-intra-mean_ratio-2}
    \end{subfigure}
    \caption{Datasets performance: AUC scores on the y-axis, $x$-axis has the datasets ordered according to increasing metric.
    }
    \label{fig:datasets_performance_vs_bag_metrics}
    \end{minipage}
\end{figure}

\begin{figure}[htb]
    \centering
    \begin{minipage}{.48\linewidth}
    \begin{subfigure}[b]{\linewidth}
        \includegraphics[width=\linewidth]{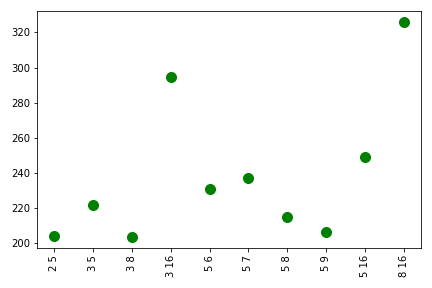}
        \caption{${\sf MeanBagSize}$}
        \label{fig:mean-bag-size-criteo-ssl}
    \end{subfigure}
    \begin{subfigure}[b]{\linewidth}
        \includegraphics[width=\linewidth]{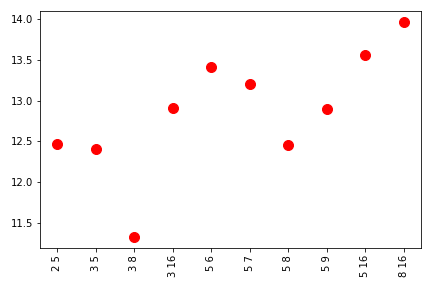}
        \caption{${\sf LabelPropStdev}$}
        \label{fig:label-prop-stdev-criteo-ssl}
    \end{subfigure}
    \begin{subfigure}[b]{\linewidth}
        \includegraphics[width=\linewidth]{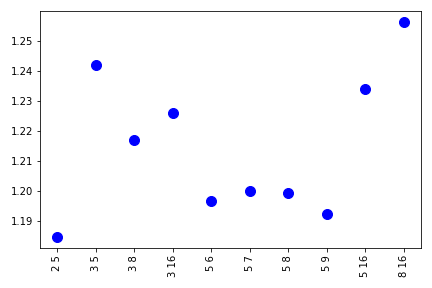}
        \caption{${\sf InterIntraRatio}$}
        \label{fig:inter-intra-mean_ratio-criteo-ssl}
    \end{subfigure}
    \caption{Criteo SSL Datasets vs. bag-level metrics: $y$-axis has the metric, $x$-axis has the datsets.
    }
    \label{fig:datasets_vs_bag_metrics_criteo_ssl}
    \end{minipage}\quad
    \begin{minipage}{.48\linewidth}
    \begin{subfigure}[b]{\linewidth}
        \includegraphics[width=\linewidth]{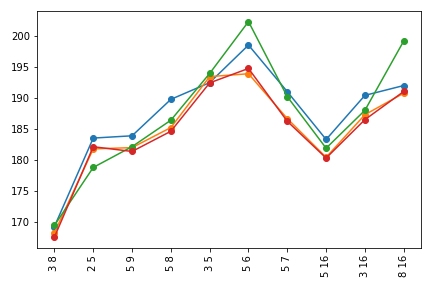}
        \caption{${\sf MeanBagSize}$}
        \label{fig:mean-bag-size-2-criteo-ssl}
    \end{subfigure}
    \begin{subfigure}[b]{\linewidth}
        \includegraphics[width=\linewidth]{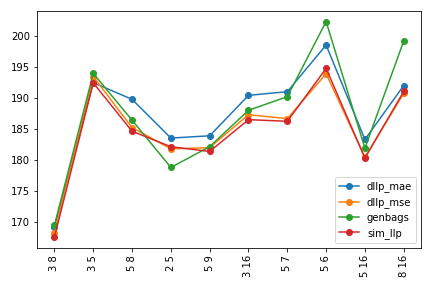}
        \caption{${\sf LabelPropStdev}$}
        \label{fig:label-prop-stdev-2-criteo-ssl}
    \end{subfigure}
    \begin{subfigure}[b]{\linewidth}
        \includegraphics[width=\linewidth]{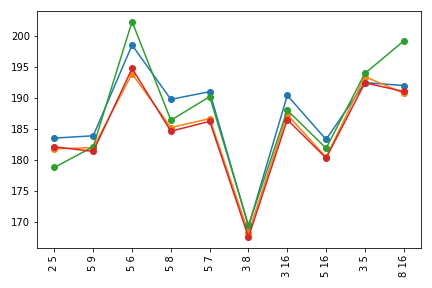}
        \caption{${\sf InterIntraRatio}$}
        \label{fig:inter-intra-mean_ratio-2-criteo-ssl}
    \end{subfigure}
    \caption{Criteo SSL Datasets performance: Test MSE on the $y$-axis, $x$-axis has the datasets ordered according to increasing metric.
    }
    \label{fig:datasets_performance_vs_bag_metrics_criteo_ssl}
    \end{minipage}
\end{figure}

\medskip
In Appendix \ref{appendix:dataset_diversity} we include the 3-D scatter plot of the the three metrics ${\sf MeanBagSize}$, ${\sf LabelPropStdev}$ and ${\sf InterIntraRatio}$ as well as three 2-D scatter plots for each of the three pairs to show that the metrics have significant variability and limited dependency with respect to each other. %

In Appendix \ref{app:Cramers} we compute the Cramer's V between grouping-key pairs $(A, B)$ and the label for each feature-bag LLP-Bench datasets from Criteo CTR. We observe that there is a significant diversity in the values which are all bounded away from $1$, indicating that the datasets are diverse in terms of the bag vs label correlations.

\section{Evaluation of baselines on LLPBench}
\label{sec:llpbench-performance}
\emph{Training and test data setup.} For each feature-grouping based dataset $(A, B)$ available after filtering in Section \ref{filter_sec}, we create the a 5 fold train/test split as follows. We first create the dataset $(A,B)$ over the entire Criteo dataset and filter out the bags as mentioned in Section \ref{sec:bag_filtering}. Using the feature-vectors in the remaining bags and their original labels, we recreate the instance-level dataset. This is then used to create a 5-fold train/test split. For each train split we recreate the bag-level training data via grouping by $(A, B)$. For the fixed-size random bag datasets, we first do a 5-fold train/test split of the entire data and partition the train splits into bags of the given fixed size. 

\emph{Training Methodology.} 
All of our baselines (listed below) are trained on the LLP datasets using minibatch training. We sample $8$ bags in each minibatch and do a forward-pass of the model on all the feature-vectors in the minibatch of bags. This allows us to obtain the predicted label proportions for each of the bags as well as the instance-level predicted-labels. Using these along with the true label proportions we compute the appropriate loss functions of the different methods. The back-propagation step then updates the model parameters.

\smallskip
We evaluate the following baseline methods on Criteo CTR LLP-Bench datasets whose methdology was explained in Section~\ref{sec:related_work}.
${\sf DLLP}$~\cite{ArdehalyC17}, ${\sf GenBags}$:    \cite{SRR},  ${\sf Easy\text{-}LLP}$\cite{easy-llp}, 
${\sf OT}$ methods:\cite{OT,DZCBV19} (${\sf Hard\text{-}EROT\text{-}LLP}$ and ${\sf Soft\text{-}EROT\text{-}LLP}$, the hard and soft entropic regularized OT methods of \cite{OT}), 
${\sf SIM\text{-}LLP}$:   \cite{KotziasDFS15} 
${\sf Mean\text{-}Map}$:  \cite{QuadriantoSCL09}. On the LLP-Bench datasets from Criteo SSCL which is a regression dataset, we evaluate the applicable baselines: ${\sf DLLP\text{-}MSE}$, ${\sf DLLP\text{-}MAE}$, ${\sf GenBags}$ and ${\sf SIM\text{-}LLP}$.

Using each of the above methods, a $2$-layer perceptron is trained using a fixed learning rate of $1e\text{-}5$. It has a multi-hot encoding layer which takes as input index encoded feature-vectors, followed by $128$ and $64$ node hidden layers respectively with {\sf relu} activation. The final output node is {\sf sigmoid} activated in case of Criteo CTR LLP datasets for which  we report the \emph{area under the ROC curve} (AUC) scores for each of the above LLP methods, while on Criteo SSCL LLP-Bench datasets MSE is the evaluation metric. We use an ${\sf Adam}$ optimizer with learning rate of $1e-5$. We also implement early-stopping with patience of $3$ epochs which monitors accuracy on the test set. Additional details of the implementation of the above methods are included in Appendix \ref{sec:additional-experiment-details}. The AUC and MSE scores for the LLP-Bench datasets and the baselines as applicable are included in Appendix \ref{appendix:baseline_results}.

\subsection{Performance of baselines on Feature Bags}\label{sec:performancevsmetrics}

\subsubsection{Criteo CTR}
Fig. \ref{fig:mean-bag-size-2} presents the trend of AUC scores for the previously described LLP methods w.r.t the ${\sf MeanBagSize}$ metric where the x-axis has the datasets ordered by increasing ${\sf MeanBagSize}$. Similarly in Figures \ref{fig:label-prop-stdev-2} and \ref{fig:inter-intra-mean_ratio-2} the datasets are ordered on the x-axis by increasing ${\sf LabelPropStdev}$ and ${\sf InterIntraRatio}$ respectively. 

First we observe that the ${\sf SIM\text{-}LLP}$, ${\sf DLLP\text{-}BCE}$, ${\sf DLLP\text{-}MSE}$, and ${\sf GenBags}$ methods are the best performing on all of the datasets with AUC scores in the $72\%$-$78\%$ range. Within them,  ${\sf SIM\text{-}LLP}$ performs the best on 41, ${\sf DLLP\text{-}MSE}$ on 6 and ${\sf DLLP\text{-}BCE}$ on 5 datasets, indicating that the additional similarity based loss helps on feature bag datasets. The AUC scores of ${\sf GenBags}$ is consistently lower than ${\sf SIM\text{-}LLP}$ and the ${\sf DLLP}$ methods, possibly due to the fact that our scenario does not have multiple bag distributions and therefore no corresponding convex programming solution to obtain the combining weights. This leads to undesirable combinations of large bags with smaller ones with roughly equal weights (ideally smaller bags should receive larger weights), leading to loss in the bag-label supervision.

On the other hand, the AUC scores of ${\sf Mean\text{-}Map}$ are the lowest for nearly all the datasets, remaining below $67\%$. The ${\sf Easy\text{-}LLP}$ and the ${\sf OT}$ methods have scores typically in the range of $65\%$-$70\%$. For reference, the same model trained on instance-level data yields around $80\%$ AUC score (see Appendix \ref{appendix:instance_level_training}).

Since the datasets are created by feature-based aggregation they may not satisfy the distributional and generative model assumptions of \cite{QuadriantoSCL09}, possibly explaining the lower scores of ${\sf Mean\text{-}Map}$. Similarly, ${\sf Easy\text{-}LLP}$ is tailored towards random-bags, and therefore may have lower scores on these datasets. The lower performance of the ${\sf OT}$ methods indicates that the pseudo-labels computed in their optimization step could significantly differ from the true labels. On the other hand, optimizing the bag-level losses as in ${\sf DLLP}$ based methods leads to more accurate model training.

The trends w.r.t. the metrics are as expected. There is a moderately decreasing trend of AUC scores with increasing ${\sf MeanBagSize}$ which is unsurprising since larger bags provide lower label supervision. More interestingly, there is a clear increasing trend with increasing ${\sf LabelPropStdev}$ along with a moderate increasing trend with increasing ${\sf InterIntraRatio}$. The latter two trends are also expected, given the explanations in Sec. \ref{sec:prelims}.

\begin{table}[]

\centering
\caption{AUC scores on Random Bags}
\begin{tabular}{ccccc}
\toprule
\multicolumn{4}{c}{\textbf{Bag size}} & \textbf{} \\
\textbf{Method} & \textit{64} & \textit{128} & \textit{256} & \textit{512} \\ \midrule
${\sf DLLP\text{-}BCE}$ & 77.54 & 76.96 & 76.24 & 75.22  \\
${\sf DLLP\text{-}MSE}$ & 77.56 & 77.03 & 76.33 & 75.42  \\
${\sf GenBags}$ & 77.08 & 76.5 & 75.75 & 75.22  \\
${\sf Easy\text{-}LLP}$ & 75.69 & 74.18 & 72.32 & 70.13  \\
${\sf OT\text{-}LLP}$ & 74.25 & 71.53 & 68.1 & 65.26  \\
${\sf SIM\text{-}LLP}$ & 77.41 & 76.73 & 75.47 & 73.13  \\
${\sf Soft\text{-}EROT\text{-}LLP}$ & 74.43 & 71.82 & 68.34 & 65.16  \\
${\sf Hard\text{-}EROT\text{-}LLP}$ & 74.27 & 71.65 & 68.08 & 65.54  \\
${\sf Mean\text{-}Map}$ & 63.17 & 63.06 & 62.83 & 62.26  \\ \bottomrule
\end{tabular}

\label{tab:rand_feat_bags}
\end{table}
\begin{table}[]
 
\centering
\caption{Range of AUC scores on feature-bag Criteo CTR datasets}
\begin{tabular}{lcc}
\toprule
\textbf{Method} & \textbf{Min AUC} & \textbf{Max AUC} \\ \midrule
${\sf DLLP\text{-}BCE}$ & 72.92 & 77.04 \\
${\sf DLLP\text{-}MSE}$ & 72.96 & 77.03 \\
${\sf GenBags}$ & 71.45 & 76.8 \\
${\sf Easy\text{-}LLP}$ & 64.59 & 71.62 \\
${\sf OT\text{-}LLP}$ & 67.57 & 72.38 \\
${\sf SIM\text{-}LLP}$ & 72.8 & 77.57 \\
${\sf Soft\text{-}EROT\text{-}LLP}$ & 67.8 & 72.5 \\
${\sf Hard\text{-}EROT\text{-}LLP}$ & 64.75 & 72.49 \\
${\sf Mean\text{-}Map}$ & 60.38 & 67.43 \\ 
\bottomrule
\end{tabular}
\label{tab:method_ranges}
\end{table}

\begin{table}[]
\caption{Dataset Statistics for Bag Datasets formed using Criteo SSCL}
\label{tab:data_stats_cssl}
\centering
\begin{tabular}{llccc}
\toprule
\textit{Col1} & \textit{Col2} & \textit{Mean Bag Size} & \textit{Std. Label Prop} & \textit{InterIntraSep} \\
\midrule
$\tilde{C}$2 & $\tilde{C}$5 & 203.91 & 12.47 & 1.18 \\
$\tilde{C}$3 & $\tilde{C}$5 & 221.79 & 12.4 & 1.24 \\
$\tilde{C}$3 & $\tilde{C}$8 & 203.48 & 11.33 & 1.22 \\
$\tilde{C}$3 & $\tilde{C}$16 & 294.38 & 12.91 & 1.23 \\
$\tilde{C}$5 & $\tilde{C}$6 & 230.63 & 13.41 & 1.2 \\
$\tilde{C}$5 & $\tilde{C}$7 & 237.24 & 13.21 & 1.2 \\
$\tilde{C}$5 & $\tilde{C}$8 & 214.65 & 12.45 & 1.2 \\
$\tilde{C}$5 & $\tilde{C}$9 & 206.6 & 12.9 & 1.19 \\
$\tilde{C}$5 & $\tilde{C}$16 & 248.85 & 13.56 & 1.23 \\
$\tilde{C}$8 & $\tilde{C}$16 & 326.17 & 13.97 & 1.26 \\
\bottomrule
\end{tabular}%
\vspace{2mm}
\end{table}
\begin{table}[]
\caption{MSE on Random Bags Datasets formed using Criteo SSCL}\label{tab:SSCL_random}
\centering
\begin{tabular}{ccccc}
\toprule
\multicolumn{4}{c}{\textbf{Bag size}} & \textbf{} \\
\textbf{Method} & \textit{64} & \textit{128} & \textit{256} & \textit{512} \\ \midrule
${\sf DLLP\text{-}MSE}$ & 159.15$\pm$\tiny{0.87} & 167.72$\pm$\tiny{1.06} & 195.13$\pm$\tiny{2.93} & 228.85$\pm$\tiny{4.42}  \\
${\sf DLLP\text{-}MAE}$ & 161.02$\pm$\tiny{0.9} & 171.66$\pm$\tiny{1.38} & 198.47$\pm$\tiny{3.91} & 232.45$\pm$\tiny{3.63}  \\
${\sf GenBags}$ & 161.52$\pm$\tiny{0.0} & 166.69$\pm$\tiny{1.07} & 172.77$\pm$\tiny{0.43} & 184.81$\pm$\tiny{1.67}  \\
${\sf SIM\text{-}LLP}$ & 159.49$\pm$\tiny{1.04} & 168.95$\pm$\tiny{1.26} & 194.1$\pm$\tiny{5.22} & 227.33$\pm$\tiny{2.18}  \\
\bottomrule
\end{tabular}
\end{table}
\begin{table}[]
\caption{Range of MSE on feature-bag datasets formed using Criteo SSCL}\label{tab:method_ranges_SSCL}
\centering
\begin{tabular}{lcc}
\toprule
\textbf{Method} & \textbf{Max MSE} & \textbf{Min MSE} \\ \midrule
${\sf DLLP\text{-}MSE}$ & 193.96 & 168.31 \\
${\sf DLLP\text{-}MAE}$ & 198.6 & 169.3 \\
${\sf GenBags}$ & 202.39 & 169.49 \\
${\sf SIM\text{-}LLP}$ & 194.82 & 167.6 \\
\bottomrule
\end{tabular}
\end{table}

\subsubsection{Criteo SSCL}
Fig. \ref{fig:mean-bag-size-2-criteo-ssl} depicts the MSE scores for the methods evaluated on the LLP datasets from  Criteo SSCL, where the x-axis has the datasets ordered by increasing ${\sf MeanBagSize}$, while Figures \ref{fig:label-prop-stdev-2-criteo-ssl} and \ref{fig:inter-intra-mean_ratio-2-criteo-ssl} have the datasets ordered by increasing ${\sf LabelPropStdev}$ and ${\sf InterIntraRatio}$ respectively.

We observe that  ${\sf SIM\text{-}LLP}$ an ${\sf DLLP\text{-}MSE}$ are the best performing methods overall, followed by ${\sf DLLP\text{-}MAE}$ and ${\sf GenBags}$. Since the comparison is w.r.t. test MSE score, it is possible that  ${\sf SIM\text{-}LLP}$ an ${\sf DLLP\text{-}MSE}$ which are MSE loss based methods optimize better for the evaluation metric. Further, as explained previously in this subsection, our scenario has just disjoint bags and no convex program is solved for ${\sf GenBags}$ possibly leading to its worse performance.

Since the number of feature bag LLP datasets from Criteo SSCL is only 10, the trends w.r.t to the bag-level metrics are not as evident as in the Criteo CTR case which has 52 such datasets. Nevertheless, we do observe a decrease in the MSE loss (i.e., increase in model performance) with increasing ${\sf MeanBagSize}$.

\subsection{Performance of Baselines on Random Bags}

\subsubsection{Criteo CTR}
Table \ref{tab:rand_feat_bags} reports the AUC scores obtained by running baselines on random bags of sizes $64, 128, 256, 512$. We observe that the performance of ${\sf Easy\text{-}LLP}$ is better random bags as compared to its performance on feature bags. As mentioned above, the performance guarantees for ${\sf Easy\text{-}LLP}$ assume random bags therefore this improvement is expected. We also notice the ${\sf SIM\text{-}LLP}$ no longer outperforms ${\sf DLLP\text{-}BCE}$ and ${\sf DLLP\text{-}MSE}$ on random bags  as it did on feature bags. The instances in feature bags are closer than those in random bags. As pairs of instances sampled in the same mini-batch are likely belong the same bag, the weight factor in the similarity loss ($\exp(\|\mathbf{x}_i - \mathbf{x}_j\|_2^2)$) is likely  to be higher in case of feature bags leading to greater supervision as compared to random bags. Further, as the random bag size increases, in  ${\sf SIM\text{-}LLP}$ due to lower label supervision the magnitude of the bag-loss could decrease, making the similarity loss more dominant which spuriously penalizes pairs of instances with different labels, leading to an overall moderation in performance.

While the performance of ${\sf Easy\text{-}LLP}$ reduces noticeably with increase in bags size, the corresponding degradation of the ${\sf OT}$ methods is particularly significant. Notably, these techniques derive instance level surrogate or pseudo-labels as training labels for the model. These trends suggest that such pseudo-labeling techniques are more severely affected by increasing bag size.

\subsubsection{Criteo SSCL}
Table \ref{tab:SSCL_random} provides the MSE scores for the ${\sf DLLP\text{-}MSE}$, ${\sf DLLP\text{-}MAE}$, ${\sf GenBags}$ and ${\sf SIM\text{-}LLP}$ methods on size $64, 128, 256, 512$ random bags datasets from Criteo SSCL. Firstly, we observe that the scores for each of the methods expectedly degrade with increasing bag size. The main takeaway however is that, while the scores for each of the methods are fairly similar for bag sizes $64$ and $128$, the ${\sf GenBags}$ method outperforms the other methods by a wide margin on bag sizes $128$ and $512$. This is unlike feature bags (see previous subsection) on which ${\sf GenBags}$ has performance similar or slightly worse than other methods. One reason could be that while the lack of  convex programming based weights for ${\sf GenBags}$ hurts on feature bags, since each random bag has the same distribution a random weight combination works much better in this case. We also do not observe such benefit with random bags on Criteo CTR, likely because the regression task on Criteo SSCL is more amenable to the ${\sf GenBags}$ methodology of linearly combining bags. 

\section{Detailed Analysis}
\label{sec:analysis}
In this section, we present more detailed analysis of various baselines and interesting feature bag datasets in the benchmark. In order to perform a subjective analysis of the datasets, we classify them -- separately for Criteo CTR and SSCL derived datasets -- based on their metric values according to the following criteria:\\
    \noindent
    \textbf{Tail size}: We perform $k$-Means clustering where each dataset is represented by a four tuple of bag size at x percentile where $x \in \{50, 70, 85, 95\}$ (See Sec \ref{sec:hardness_metrics}). For Criteo CTR we take $k = 4$ and name these clusters as \textit{very short-tailed}, \textit{short-tailed}, \textit{long-tailed} and \textit{very long-tailed} in increasing order of mean bag size at 70 percentile of each cluster. On Criteo SSCL we take $k = 3$ and similarly name the clusters as \textit{short-tailed}, \textit{medium-tailed} and \textit{long-tailed}. \\ %
    \noindent
    \textbf{Label Variation}: For Criteo CTR, we perform 4-Means clustering with each dataset represented by ${\sf LabelPropStdev}$. We classify the datasets as \textit{low}, \textit{medium}, \textit{high} and \textit{very high} ${\sf LabelPropStdev}$ dataset in increasing order of the mean ${\sf LabelPropStdev}$ of the cluster. On Criteo SSCL 3-Means clustering is done and the corresponding labels are \textit{low}, \textit{medium} and \textit{high}.\\ %
    \noindent
    \textbf{Bag Separation}: We perform 4-Means clustering on Criteo CTR derived datasets using ${\sf InterIntraRatio}$ classifying them as \textit{less-separated}, \textit{medium-separated}, \textit{well-separated} and \textit{far-separated} in increasing order of mean ${\sf InterIntraRatio}$ of each cluster. As above, for Criteo SSCL datasets 3-Means clustering is done and clusters are annotated \textit{less-separated}, \textit{medium-separated} and \textit{well-separated}. 
    
\noindent    
The detailed classifications based on mean bag size, ${\sf LabelPropStdev}$ and ${\sf InterIntraRatio}$  of all LLP-Bench datasets  are respectively reported Appendices \ref{appendix:bag_size_dist}, \ref{appendix:std_label_prop_dist} and \ref{appendix:bag_sep_stats}.

\subsection{Best and worst performance of every baseline}
Tables \ref{tab:method_ranges} and \ref{tab:method_ranges_SSCL} list the range of scores of the methods evaluated on the Criteo CTR (AUC) and Criteo SSCL (MSE) feature bag datasets respectively.  
From these a range of AUC scores of at least $4$ percentage points for all the baselines on Criteo CTR, and an MSE range of 12-15\% of the maximum MSE for each baseline on Criteo SSCL datasets is observed. This shows that LLP-Bench has enough diversity in the underlying datasets and that it can be used to find opportunities to improve SOTA algorithms to further LLP research on Criteo CTR/SSCL and other tabular datasets. 

{\bf Criteo CTR.} 
We observe that all baselines perform the best on the dataset $(\tn{C}3, \tn{C}11)$ and perform the worst on the dataset $(\tn{C}6, \tn{C}10)$.   $(\tn{C}3, \tn{C}11)$ is a \textit{very short-tailed} dataset with \textit{high} standard deviation in the label proportion and it \textit{well-separated}. This makes this dataset relatively easy to learn on.  On the other hand, $(\tn{C}6, \tn{C}10)$ is a \textit{long-tailed} dataset with \textit{very low} standard deviation in the label proportion and it \textit{Less-Separated}. This dataset represents the worst-case scenario of the three feature combination making it very hard to learn on. 

{\bf Criteo SSCL.} Here as well there is one dataset $(\tilde{C}3,\tilde{C}8)$ having the best performance and one dataset $(\tilde{C}5,\tilde{C}6)$ giving the worst performance, for all the three baselines. The former -- $(\tilde{C}3,\tilde{C}8)$ -- has the lowest ${\sf MeanBagSize}$ and also has \emph{short-tailed} bag sizes making it relatively easier to learn on, while $(\tilde{C}5,\tilde{C}6)$ is \textit{less-separated} with \textit{high} ${\sf LabelPropStdev}$ which could make it a relatively intractable LLP dataset.

\subsection{Dataset Analysis}
In this section, we select a few datasets that do not perform as expected in the trend-lines of Figure~\ref{fig:datasets_performance_vs_bag_metrics}. We analyse them further to explain the performance metrics that we observe. 

On the LLP datasets from Criteo CTR we have the following:\\
\noindent
    {\bf 1.} Datasets $(\tn{C}4, \tn{C}15)$ and $(\tn{C}4, \tn{C}10)$ are \textit{medium-separated} but they perform better than other such datasets (Fig \ref{fig:inter-intra-mean_ratio-2}). This is because they are \textit{short-tailed} and \textit{very short-tailed} respectively.\\
    \noindent
    {\bf 2.} Datasets $(\tn{C}7, \tn{C}8)$ and $(\tn{C}1, \tn{C}7)$ are \textit{well-separated} but they perform worse than other such datasets (Fig \ref{fig:inter-intra-mean_ratio-2}) because they have \textit{very low} ${\sf LabelPropStddev}$ and are \textit{long-tailed}. Similarly, datasets $(\tn{C}6, \tn{C}7), (\tn{C}7, \tn{C}14)$ and $(\tn{C}7, \tn{C}20)$ are all \textit{well-separated} and yet they perform poorly that other such datasets. They are all \textit{very long-tailed} and have \textit{very low} ${\sf LabelPropStdDev}$ (except $(\tn{C}7, \tn{C}14)$ which has a \textit{low} ${\sf LabelPropStdDev}$). \\
    \noindent
    {\bf 3.} It can also be observed from Fig \ref{fig:label-prop-stdev-2} that the datasets $(\tn{C}7, \tn{C}20), (\tn{C}1, \tn{C}7)$ and $(\tn{C}7, \tn{C}8)$ actually perform better as compared to other \textit{very low} ${\sf LabelPropStdDev}$ datasets because they are \textit{long-tailed} and \textit{well-separated} (except $(\tn{C}7, \tn{C}20)$ which is \textit{very long-tailed})\\
    \noindent
    {\bf 4.} Dataset $(\tn{C}4, \tn{C}15)$ performs poorly as compared to other \textit{medium} ${\sf LabelPropStddev}$ datasets even when it is \textit{short-tailed} as it is \textit{less-separated} (Fig \ref{fig:label-prop-stdev-2})\\
    \noindent
    {\bf 5.} Dataset $(\tn{C}7, \tn{C}26)$ is really interesting since it performs poorly as compared to other \textit{medium} ${\sf LabelPropStddev}$, \textit{far-separated} and \textit{short-tailed} datasets. \\
    \noindent
    {\bf 6.} Datasets $(\tn{C}2, \tn{C}11)$ and $(\tn{C}2, \tn{C}13)$ perform slightly better than other datasets with comparable ${\sf MeanBagSize}$ (Fig \ref{fig:mean-bag-size-2}). They are both \textit{long-tailed}, with \textit{low} ${\sf LabelPropStddev}$ and \textit{medium-separated} and hence their higher performance cannot be explained based on our metrics.
    
\noindent
On Criteo SSCL dataset we observe that ${\sf GenBags}$ performs significantly worse than the other baselines on $(\tilde{C}8,\tilde{C}16)$. This has the highest ${\sf InterIntraRatio}$ and is the only \textit{well-separated} LLP dataset from Criteo SSCL. Since ${\sf GenBags}$ takes linear combinations of bags it could diminish this well-separatedness of he bags, while the other baselines preserve it, leading to a comparative degradation of ${\sf GenBags}$.

\medskip
\noindent
{\bf Experimental Code.} The code for the experiments included in this paper is available at \url{https://github.com/google-research/google-research/tree/master/LLP_Bench}

\section{Conclusion}

We present the design of  LLP-Bench: a diverse collection of tabular LLP datasets from the Criteo CTR and SSCL datasets as a benchmark for evaluating LLP techniques on binary claasification and regression tasks. In this process, our work analyzes  bag collections given by grouping on at most two categorical features, based on their distribution of bags as well as label proportions. We show that LLP-Bench has significant diversity in the nature of the datasets that are present in it. To the best of our knowledge, LLP-Bench  is the first large scale tabular benchmark with extensive diversity in the underlying datasets. 
We presented a detailed analysis of 9 SOTA baselines on LLP-Bench and explained their performance by correlating it with the underlying dataset characteristics. Again, to the best of our knowledge no other study has compared SOTA techniques to this level of detail that we have. 
We believe our work addresses to a great extent the current lack of a large scale tabular LLP benchmark.  LLP-Bench along with the four dataset hardness metrics can be used to systematically study and design new  LLP techniques.

{\bf Limitations and Future Work.}
While the performance and outlier analysis (Sections \ref{sec:performancevsmetrics} and \ref{sec:analysis}) use up to four metrics, a deeper explanation of outliers could be possible using additional metrics. Future work could also incorporate more LLP algorithms as well as additional model architectures.

\FloatBarrier

\bibliographystyle{plainnat}
\bibliography{references}

\begin{thebibliography}{43}
\providecommand{\natexlab}[1]{#1}
\providecommand{\url}[1]{\texttt{#1}}
\expandafter\ifx\csname urlstyle\endcsname\relax
  \providecommand{\doi}[1]{doi: #1}\else
  \providecommand{\doi}{doi: \begingroup \urlstyle{rm}\Url}\fi

\bibitem[Ardehaly and Culotta(2017)]{ArdehalyC17}
Ehsan~Mohammady Ardehaly and Aron Culotta.
\newblock Co-training for demographic classification using deep learning from
  label proportions.
\newblock In \emph{{ICDM}}, pages 1017--1024, 2017.

\bibitem[Bortsova et~al.(2018{\natexlab{a}})Bortsova, Dubost, {\O}rting,
  Katramados, Hogeweg, Thomsen, Wille, and de~Bruijne]{bortsova2018deep}
Gerda Bortsova, Florian Dubost, Silas {\O}rting, Ioannis Katramados, Laurens
  Hogeweg, Laura Thomsen, Mathilde Wille, and Marleen de~Bruijne.
\newblock Deep learning from label proportions for emphysema quantification.
\newblock In \emph{Medical Image Computing and Computer Assisted
  Intervention--MICCAI 2018: 21st International Conference, Granada, Spain,
  September 16-20, 2018, Proceedings, Part II 11}, pages 768--776. Springer,
  2018{\natexlab{a}}.

\bibitem[Bortsova et~al.(2018{\natexlab{b}})Bortsova, Dubost, {\O}rting,
  Katramados, Hogeweg, Thomsen, Wille, and de~Bruijne]{Bortsova18}
Gerda Bortsova, Florian Dubost, Silas~N. {\O}rting, Ioannis Katramados, Laurens
  Hogeweg, Laura~H. Thomsen, Mathilde M.~W. Wille, and Marleen de~Bruijne.
\newblock Deep learning from label proportions for emphysema quantification.
\newblock In \emph{{MICCAI}}, volume 11071 of \emph{Lecture Notes in Computer
  Science}, pages 768--776. Springer, 2018{\natexlab{b}}.
\newblock URL \url{https://arxiv.org/abs/1807.08601}.

\bibitem[Busa{-}Fekete et~al.(2023)Busa{-}Fekete, Choi, Dick, Gentile, and
  {Mu{\~{n}}oz Medina}]{easy-llp}
R{\'{o}}bert~Istvan Busa{-}Fekete, Heejin Choi, Travis Dick, Claudio Gentile,
  and Andr{\'{e}}s {Mu{\~{n}}oz Medina}.
\newblock Easy learning from label proportions.
\newblock \emph{CoRR}, abs/2302.03115, 2023.
\newblock \doi{10.48550/arXiv.2302.03115}.
\newblock URL \url{https://doi.org/10.48550/arXiv.2302.03115}.

\bibitem[Chen et~al.(2004)Chen, Huang, and Ramakrishnan]{chen2004cost}
Lei Chen, Zheng Huang, and Raghu Ramakrishnan.
\newblock Cost-based labeling of groups of mass spectra.
\newblock In \emph{Proceedings of the 2004 ACM SIGMOD international conference
  on Management of data}, pages 167--178, 2004.

\bibitem[Chen et~al.(2023{\natexlab{a}})Chen, Fu, Karbasi, and
  Mirrokni]{CFKM23}
Lin Chen, Gang Fu, Amin Karbasi, and Vahab Mirrokni.
\newblock Learning from aggregated data: Curated bags versus random bags.
\newblock \emph{CoRR}, abs/2305.09557, 2023{\natexlab{a}}.
\newblock \doi{10.48550/arXiv.2305.09557}.
\newblock URL \url{https://doi.org/10.48550/arXiv.2305.09557}.

\bibitem[Chen et~al.(2023{\natexlab{b}})Chen, Fu, Karbasi, and
  Mirrokni]{chen2023learning}
Lin Chen, Thomas Fu, Amin Karbasi, and Vahab Mirrokni.
\newblock Learning from aggregated data: Curated bags versus random bags.
\newblock \emph{arXiv preprint arXiv:2305.09557}, 2023{\natexlab{b}}.

\bibitem[Chen et~al.(2009)Chen, Liu, Qian, and Zhang]{ChenLQZ09}
Shuo Chen, Bin Liu, Mingjie Qian, and Changshui Zhang.
\newblock Kernel k-means based framework for aggregate outputs classification.
\newblock In Y{\"{u}}cel Saygin, Jeffrey~Xu Yu, Hillol Kargupta, Wei Wang,
  Sanjay Ranka, Philip~S. Yu, and Xindong Wu, editors, \emph{{ICDM}}, pages
  356--361, 2009.

\bibitem[Criteo(2014)]{Criteo:2014}
Criteo.
\newblock Kaggle display advertising challenge dataset, 2014.
\newblock URL
  \url{http://labs.criteo.com/2014/02/kaggle-display-advertising-challenge-dataset/}.

\bibitem[de~Freitas and K{\"{u}}ck(2005)]{FreitasK05}
Nando de~Freitas and Hendrik K{\"{u}}ck.
\newblock Learning about individuals from group statistics.
\newblock In \emph{{UAI}}, pages 332--339, 2005.

\bibitem[Dery et~al.(2017)Dery, Nachman, Rubbo, and Schwartzman]{DNRS}
L.~M. Dery, B.~Nachman, F.~Rubbo, and A.~Schwartzman.
\newblock Weakly supervised classification in high energy physics.
\newblock \emph{Journal of High Energy Physics}, 2017\penalty0 (5):\penalty0
  1--11, 2017.

\bibitem[Dua and Graff(2017)]{Dua:2019}
Dheeru Dua and Casey Graff.
\newblock {UCI} machine learning repository, 2017.
\newblock URL \url{http://archive.ics.uci.edu/ml}.

\bibitem[Dulac{-}Arnold et~al.(2019)Dulac{-}Arnold, Zeghidour, Cuturi, Beyer,
  and Vert]{DZCBV19}
Gabriel Dulac{-}Arnold, Neil Zeghidour, Marco Cuturi, Lucas Beyer, and
  Jean{-}Philippe Vert.
\newblock Deep multi-class learning from label proportions.
\newblock \emph{CoRR}, abs/1905.12909, 2019.
\newblock URL \url{http://arxiv.org/abs/1905.12909}.

\bibitem[Dulac-Arnold et~al.(2019)Dulac-Arnold, Zeghidour, Cuturi, Beyer, and
  Vert]{dulac2019deep}
Gabriel Dulac-Arnold, Neil Zeghidour, Marco Cuturi, Lucas Beyer, and
  Jean-Philippe Vert.
\newblock Deep multi-class learning from label proportions.
\newblock \emph{arXiv preprint arXiv:1905.12909}, 2019.

\bibitem[Florencio and Herley(2007)]{florencio2007large}
Dinei Florencio and Cormac Herley.
\newblock A large-scale study of web password habits.
\newblock In \emph{Proceedings of the 16th international conference on World
  Wide Web}, pages 657--666, 2007.

\bibitem[He et~al.(2014)He, Pan, Jin, Xu, Liu, Xu, Shi, Atallah, Herbrich,
  Bowers, et~al.]{he2014practical}
Xinran He, Junfeng Pan, Ou~Jin, Tianbing Xu, Bo~Liu, Tao Xu, Yanxin Shi,
  Antoine Atallah, Ralf Herbrich, Stuart Bowers, et~al.
\newblock Practical lessons from predicting clicks on ads at facebook.
\newblock In \emph{Proceedings of the eighth international workshop on data
  mining for online advertising}, pages 1--9, 2014.

\bibitem[Hern{\'{a}}ndez{-}Gonz{\'{a}}lez
  et~al.(2013)Hern{\'{a}}ndez{-}Gonz{\'{a}}lez, Inza, and Lozano]{HG}
Jer{\'{o}}nimo Hern{\'{a}}ndez{-}Gonz{\'{a}}lez, I{\~{n}}aki Inza, and
  Jos{\'{e}}~Antonio Lozano.
\newblock Learning bayesian network classifiers from label proportions.
\newblock \emph{Pattern Recognit.}, 46\penalty0 (12):\penalty0 3425--3440,
  2013.

\bibitem[Hernández-González et~al.(2018)Hernández-González, Inza,
  Crisol-Ortíz, Guembe, Iñarra, and Lozano]{hernandez2018}
Jerónimo Hernández-González, Iñaki Inza, Lorena Crisol-Ortíz, María~A.
  Guembe, María~J Iñarra, and Jose~A Lozano.
\newblock Fitting the data from embryo implantation prediction: Learning from
  label proportions.
\newblock \emph{Statistical methods in medical research}, 27\penalty0
  (4):\penalty0 1056--1066, 2018.

\bibitem[Kotzias et~al.(2015)Kotzias, Denil, de~Freitas, and
  Smyth]{KotziasDFS15}
Dimitrios Kotzias, Misha Denil, Nando de~Freitas, and Padhraic Smyth.
\newblock From group to individual labels using deep features.
\newblock In \emph{SIGKDD}, pages 597--606, 2015.

\bibitem[Liu et~al.(2019{\natexlab{a}})Liu, Wang, Qi, Tian, and Shi]{LiuWQTS19}
Jiabin Liu, Bo~Wang, Zhiquan Qi, Yingjie Tian, and Yong Shi.
\newblock Learning from label proportions with generative adversarial networks.
\newblock In \emph{{NeurIPS}}, pages 7167--7177, 2019{\natexlab{a}}.

\bibitem[Liu et~al.(2019{\natexlab{b}})Liu, Wang, Qi, Tian, and
  Shi]{liu2019learning}
Jiabin Liu, Bo~Wang, Zhiquan Qi, Yingjie Tian, and Yong Shi.
\newblock Learning from label proportions with generative adversarial networks.
\newblock \emph{Advances in neural information processing systems}, 32,
  2019{\natexlab{b}}.

\bibitem[Liu et~al.(2021)Liu, Wang, Shen, Qi, and Tian]{OT}
Jiabin Liu, Bo~Wang, Xin Shen, Zhiquan Qi, and Yingjie Tian.
\newblock Two-stage training for learning from label proportions.
\newblock In Zhi{-}Hua Zhou, editor, \emph{Proc. {IJCAI}}, pages 2737--2743,
  2021.

\bibitem[Liu et~al.(2022)Liu, Qi, Wang, Tian, and Shi]{liu2022self}
Jiabin Liu, Zhiquan Qi, Bo~Wang, YingJie Tian, and Yong Shi.
\newblock Self-llp: Self-supervised learning from label proportions with
  self-ensemble.
\newblock \emph{Pattern Recognition}, 129:\penalty0 108767, 2022.

\bibitem[McMahan et~al.(2013)McMahan, Holt, Sculley, Young, Ebner, Grady, Nie,
  Phillips, Davydov, Golovin, et~al.]{mcmahan2013ad}
H~Brendan McMahan, Gary Holt, David Sculley, Michael Young, Dietmar Ebner,
  Julian Grady, Lan Nie, Todd Phillips, Eugene Davydov, Daniel Golovin, et~al.
\newblock Ad click prediction: a view from the trenches.
\newblock In \emph{Proceedings of the 19th ACM SIGKDD international conference
  on Knowledge discovery and data mining}, pages 1222--1230, 2013.

\bibitem[Musicant et~al.(2007)Musicant, Christensen, and Olson]{Musicant}
David~R. Musicant, Janara~M. Christensen, and Jamie~F. Olson.
\newblock Supervised learning by training on aggregate outputs.
\newblock In \emph{{ICDM}}, pages 252--261. {IEEE} Computer Society, 2007.

\bibitem[O'Brien et~al.(2022)O'Brien, Thiagarajan, Das, Barreto, Verma, Hsu,
  Neufeld, and Hunt]{DBLP:journals/corr/abs-2201-12666}
Conor O'Brien, Arvind Thiagarajan, Sourav Das, Rafael Barreto, Chetan Verma,
  Tim Hsu, James Neufeld, and Jonathan~J. Hunt.
\newblock Challenges and approaches to privacy preserving post-click conversion
  prediction.
\newblock \emph{CoRR}, abs/2201.12666, 2022.
\newblock URL \url{https://arxiv.org/abs/2201.12666}.

\bibitem[{\O}rting et~al.(2016){\O}rting, Petersen, Wille, Thomsen, and {de
  Bruijne}]{Orting16}
{Silas Nyboe} {\O}rting, Jens Petersen, Mathilde Wille, Laura Thomsen, and
  Marleen {de Bruijne}.
\newblock Quantifying emphysema extent from weakly labeled ct scans of the
  lungs using label proportions learning.
\newblock In \emph{The Sixth International Workshop on Pulmonary Image
  Analysis}, pages 31--42, 2016.

\bibitem[Patrini et~al.(2014)Patrini, Nock, Caetano, and Rivera]{PatriniNCR14}
Giorgio Patrini, Richard Nock, Tib{\'{e}}rio~S. Caetano, and Paul Rivera.
\newblock (almost) no label no cry.
\newblock In Zoubin Ghahramani, Max Welling, Corinna Cortes, Neil~D. Lawrence,
  and Kilian~Q. Weinberger, editors, \emph{Advances in Neural Information
  Processing Systems}, pages 190--198, 2014.

\bibitem[Quadrianto et~al.(2009)Quadrianto, Smola, Caetano, and
  Le]{QuadriantoSCL09}
Novi Quadrianto, Alexander~J. Smola, Tib{\'{e}}rio~S. Caetano, and Quoc~V. Le.
\newblock Estimating labels from label proportions.
\newblock \emph{J. Mach. Learn. Res.}, 10:\penalty0 2349--2374, 2009.

\bibitem[R{\"{u}}ping(2010)]{Ruping10}
Stefan R{\"{u}}ping.
\newblock {SVM} classifier estimation from group probabilities.
\newblock In Johannes F{\"{u}}rnkranz and Thorsten Joachims, editors,
  \emph{{ICML}}, pages 911--918, 2010.

\bibitem[Saket(2021)]{Saket21}
Rishi Saket.
\newblock Learnability of linear thresholds from label proportions.
\newblock In \emph{{NeurIPS}}, pages 6555--6566, 2021.

\bibitem[Saket(2022)]{Saket22}
Rishi Saket.
\newblock Algorithms and hardness for learning linear thresholds from label
  proportions.
\newblock In \emph{{NeurIPS}}, 2022.

\bibitem[Saket et~al.(2022)Saket, Raghuveer, and Ravindran]{SRR}
Rishi Saket, Aravindan Raghuveer, and Balaraman Ravindran.
\newblock On combining bags to better learn from label proportions.
\newblock In \emph{{AISTATS}}, volume 151 of \emph{Proceedings of Machine
  Learning Research}, pages 5913--5927. {PMLR}, 2022.
\newblock URL \url{https://proceedings.mlr.press/v151/saket22a.html}.

\bibitem[Scott and Zhang(2020)]{MutCon}
Clayton Scott and Jianxin Zhang.
\newblock Learning from label proportions: {A} mutual contamination framework.
\newblock In \emph{NeurIPS}, 2020.

\bibitem[Song et~al.(2019)Song, Shi, Xiao, Duan, Xu, Zhang, and
  Tang]{autoint_cikm_2019}
Weiping Song, Chence Shi, Zhiping Xiao, Zhijian Duan, Yewen Xu, Ming Zhang, and
  Jian Tang.
\newblock Autoint: Automatic feature interaction learning via self-attentive
  neural networks.
\newblock In \emph{CIKM}, 2019.

\bibitem[Stolpe and Morik(2011)]{StolpeM11}
Marco Stolpe and Katharina Morik.
\newblock Learning from label proportions by optimizing cluster model
  selection.
\newblock In Dimitrios Gunopulos, Thomas Hofmann, Donato Malerba, and Michalis
  Vazirgiannis, editors, \emph{{ECML} {PKDD} Proceedings, Part {III}}, volume
  6913, pages 349--364. Springer, 2011.

\bibitem[Tallis and Yadav(2018)]{tallis2018reacting}
Marcelo Tallis and Pranjul Yadav.
\newblock Reacting to variations in product demand: An application for
  conversion rate ({CR}) prediction in sponsored search.
\newblock In \emph{2018 IEEE International Conference on Big Data (Big Data)},
  pages 1856--1864. IEEE, 2018.

\bibitem[Tsai and Lin(2020)]{tsai2020learning}
Kuen-Han Tsai and Hsuan-Tien Lin.
\newblock Learning from label proportions with consistency regularization.
\newblock In \emph{Asian Conference on Machine Learning}, pages 513--528. PMLR,
  2020.

\bibitem[Wagstaff and Lane(2007)]{WL07}
K.~L. Wagstaff and T.~Lane.
\newblock Salience assignment for multiple-instance regression.
\newblock In \emph{Workshop on Constrained Optimization and Structured Output
  {(ICML)}}, 2007.

\bibitem[Wang et~al.(2008)Wang, Radosavljevic, Han, Obradovic, and
  Vucetic]{WRHOV08}
Z.~Wang, V.~Radosavljevic, B.~Han, Z.~Obradovic, and S.~Vucetic.
\newblock \emph{Aerosol Optical Depth Prediction from Satellite Observations by
  Multiple Instance Regression}, pages 165--176.
\newblock 2008.

\bibitem[Wojtusiak et~al.(2011)Wojtusiak, Irvin, Birerdinc, and Baranova]{WIBB}
J.~Wojtusiak, K.~Irvin, A.~Birerdinc, and A.~V. Baranova.
\newblock Using published medical results and non-homogenous data in rule
  learning.
\newblock In \emph{Proc. International Conference on Machine Learning and
  Applications and Workshops}, volume~2, pages 84--89. IEEE, 2011.

\bibitem[Yu et~al.(2013)Yu, Liu, Kumar, Jebara, and Chang]{YuLKJC13}
Felix~X. Yu, Dong Liu, Sanjiv Kumar, Tony Jebara, and Shih{-}Fu Chang.
\newblock $\propto${SVM} for learning with label proportions.
\newblock In \emph{ICML}, volume~28 of \emph{{JMLR} Workshop and Conference
  Proceedings}, pages 504--512, 2013.

\bibitem[Zhang et~al.(2022)Zhang, Wang, and Scott]{zhang2022learning}
Jianxin Zhang, Yutong Wang, and Clay Scott.
\newblock Learning from label proportions by learning with label noise.
\newblock \emph{Advances in Neural Information Processing Systems},
  35:\penalty0 26933--26942, 2022.

\end{thebibliography}

\newpage

\appendix

\section{Proofs of Lemmas and Algorithms}

While ${\sf BagSep}$ is not a metric since ${\sf BagSep}(B, B)$ is not necessarily zero, the following lemma (proved in Appendix \ref{sec:bagseplemma}) shows that it does satisfy the other metric properties.
\begin{lemma}
${\sf BagSep}$ satisfies non-negativity, symmetry and triangle inequality.
\end{lemma}

We have the following lemma proved in Appendix \ref{appendix:meanbag}.
\begin{lemma} \label{lem:meanbag}
For any bag $B$, (i) ${\sf InterBagSep}(B, d)/{\sf BagSep}(B, B, d) \geq 1/2$ when $d$ is a metric, (ii) ${\sf InterBagSep}(B, d)/{\sf BagSep}(B, B, d) \geq 1/4$ when $d$ is the $\ell_2^2$ distance.
\end{lemma}
The following is a straightforward corollary of Lemma \ref{lem:meanbag}.
\begin{corollary}
(i) When $d$ is a metric: ${\sf MeanInterBagSep}(\mc{B}, d)/{\sf MeanIntraBagSep}(\mc{B}, d) \geq 1/2.$ 
(ii) When $d$ is the $\ell_2^2$ distance: ${\sf MeanInterBagSep}(\mc{B}, d)/{\sf MeanIntraBagSep}(\mc{B}, d) \geq 1/4.$ 
\end{corollary}
We expect this ratio to achieve values substantially less than $1$ in adversarial cases. Appendix \ref{sec:example} provides an example of such a case. For convenience, for $\mc{B}$, we use ${\sf InterIntraRatio}$ to denote  ${\sf MeanInterBagSep}(\mc{B}, d)/{\sf MeanIntraBagSep}(\mc{B}, d)$ when $d = \ell_2^2$.

\subsection{Proof of Lemma A.1}
\label{sec:bagseplemma}
\begin{proof} From Def. \ref{def:bagsep}, the non-negativity and symmetry properties are obvious.

\textbf{Triangle Inequality} : let $B_1, B_2, B_3 \in \mathcal{B}$, and we use the following notation for convenience: $B_1 = \{x_i | i \in [n]\}$, $B_2 = \{y_j | j \in [m]\}$, $B_3 = \{z_k | k \in [l]\}$.
As $d$ is a metric, we know that for all $i \in [n], j \in [m]$ and $k \in [l]$, $d(x_i, z_k) \leq d(x_i, y_j) + d(y_j, z_k)$.
Hence, 
\begin{eqnarray}
& & d(x_i, z_k) \leq \frac{\sum_{j=1}^{j=m}d(x_i, y_j)}{m} + \frac{\sum_{j=1}^{j=m}d(y_j, z_k)}{m} \nonumber \\
&\Rightarrow &\frac{\sum_{i=1}^{i=n}d(x_i, z_k)}{n} \leq \frac{\sum_{i=1}^{i=n}\sum_{j=1}^{j=m}d(x_i, y_j)}{nm} + \frac{\sum_{j=1}^{j=m}d(y_j, z_k)}{m} \nonumber \\
& \Rightarrow & \frac{\sum_{k=1}^{k=l}\sum_{i=1}^{i=n}d(x_i, z_k)}{ln} \leq \frac{\sum_{i=1}^{i=n}\sum_{j=1}^{j=m}d(x_i, y_j)}{nm} + \frac{\sum_{k=1}^{k=l}\sum_{j=1}^{j=m}d(y_j, z_k)}{ml}\nonumber \\
& \Rightarrow & {\sf BagSep}(B_1, B_3, d) \leq {\sf BagSep}(B_1, B_2, d) + {\sf BagSep}(B_2, B_3, d) \nonumber
\end{eqnarray}
\end{proof}

\subsection{Proof of Lemma A.2}
\label{appendix:meanbag}
\begin{proof}
Let $B \in \mathcal{B}$. Using triangle inequality and symmetry from \textit{Lemma A.1}:
\begin{eqnarray}
& & \forall B' \in \mc{B}, {\sf BagSep}(B, B, d) \leq {\sf BagSep}(B, B', d) + {\sf BagSep}(B', B, d) \nonumber \\
&\Rightarrow &\forall B' \in \mc{B}, {\sf BagSep}(B, B, d) \leq 2{\sf BagSep}(B', B, d) \nonumber \\
& \Rightarrow & {\sf BagSep}(B, B, d) \leq 2\frac{\sum_{B' \in \mathcal{B}, B' \not= B}{\sf BagSep}(B', B, d)}{|\mathcal{B}|-1}\nonumber \\
& \Rightarrow & {\sf BagSep}(B, B, d) \leq 2{\sf InterBagSep}(B, d) \nonumber \\
& \Rightarrow & {\sf InterBagSep}(B, d)/{\sf BagSep}(B, B, d) \geq 1/2 \nonumber
\end{eqnarray}
\end{proof}
The squared euclidean distance is not a metric as it follows all properties other than the triangle inequality. Hence, we show the following
\begin{lemma}
For any $a, b \in R^n$, $\frac{1}{2}||a+b||_2^2 \leq ||a||_2^2 + ||b||_2^2$
\end{lemma}
\begin{theorem}
Given $X$, $Y$ and $\mathcal{B}$, for any $B_1, B_2, B_3 \in \mathcal{B}$,
$$\frac{1}{2}{\sf BagSep}(B_1, B_3, \ell_2^2) \leq {\sf BagSep}(B_1, B_2, \ell_2^2) + {\sf BagSep}(B_2, B_3, \ell_2^2)$$
\end{theorem}
\begin{proof}
Follows by replacing triangle inequality in \textit{Lemma A.1} with inequality in \textit{Lemma A.4}
\end{proof}
\begin{corollary}
${\sf InterBagSep}(B, \ell_2^2)/{\sf BagSep}(B, B, \ell_2^2) \geq 1/4$
\end{corollary}
\begin{proof}
Follows by replacing inequality in proof of \textit{Lemma A.2} with inequality in \textit{Theorem A.5}
\end{proof}
\subsection{Proof of Corollary A.3}
\label{appendix:a3}
\begin{proof}
Given $X$, $Y$ and $\mathcal{B}$, and metric $d$ in $R^n$. 
Starting with inequality in \textit{Lemma A.2}\\
\begin{eqnarray}
& & \forall B \in \mc{B}, {\sf BagSep}(B, B, d) \leq 2{\sf InterBagSep}(B, d) \nonumber \\
&\Rightarrow &\sum_{B \in \mathcal{B}}{\sf BagSep}(B, B, d) \leq 2\sum_{B \in \mathcal{B}}{\sf InterBagSep}(B, d) \nonumber \\
& \Rightarrow & \frac{1}{|\mathcal{B}|}{\sf BagSep}(B, B, d) \leq 2\frac{1}{|\mathcal{B}|}{\sf InterBagSep}(B, d)\nonumber \\
& \Rightarrow & {\sf MeanInterBagSep}(\mc{B}, d)/{\sf MeanIntraBagSep}(\mc{B}, d) \geq 1/2 \nonumber 
\end{eqnarray}
Starting with inequality for $\ell_2^2$-distance in \textit{Lemma A.2}, we get\\
${\sf MeanInterBagSep}(\mc{B}, \ell_2^2)/{\sf MeanIntraBagSep}(\mc{B}, \ell_2^2) \geq 1/2$
\end{proof}
\subsection{Bag Distance Results using squared euclidean distance}
\label{appendix:sq_l2_bag_sep}
We use the squared euclidean distance to compute the bag distances as it makes the computation faster. Algorithm \ref{alg:one} is used to compute the Bag Separation for any general metric $d$.
\begin{algorithm}
\caption{Compute Bag Separation of a dataset}\label{alg:one}
\KwData{Set of bags $\mathcal{B}$, metric $d$ on $R^n$}
\KwResult{${\sf BagSepMatrix}(\mc{B}, d)$}
${\sf BagSepMatrix} \gets [0]_{|\mathcal{B}|x|\mathcal{B}|}$\\
\For{$B_1 \in \mathcal{B}$}{
    \For{$B_2 \in \mathcal{B}$} {
        \For{$i \in B_1$}{
            \For{$j \in B_2$}{
                ${\sf BagSepMatrix}[B_1, B_2] \gets {\sf BagSepMatrix}[B_1, B_2] + d(x^{(i)}, x^{(j)})$
            }
        }
        ${\sf BagSepMatrix}[B_1, B_2] \gets {\sf BagSepMatrix}[B_1, B_2]/(|B_1||B_2|)$ 
    }
}
\end{algorithm}
\begin{theorem}
Assuming the Bags to be disjoint, the running time of Algorithm \ref{alg:one} is $O(m^2n)$ where $m$ is the number of examples and n is the dimension of the input space.
\end{theorem}
\begin{proof}
Runtime = $\sum_{B_1 \in \mathcal{B}}\sum_{B_2 \in \mathcal{B}}|B_1||B_2|n = m^2n$
\end{proof}
Now, this computation can be simplified due to the following. Let $\|B\| := \frac{1}{|B|}\sum_{x \in B}\|x\|_2^2$ and $\mu(B) := \frac{1}{|B|}\sum_{x \in B}\,x$\\
\begin{lemma}
\label{lemma:ell_2_case}
For any $B, B' \in \mathcal{B}$, 
$${\sf BagSep}(B, B', \ell_2^2) = \|B\| + \|B'\| - 2 \langle \mu(B), \mu(B') \rangle$$
\end{lemma}
\begin{proof}
Let $B = \{x_i | i \in [n]\}, B' = \{y_j | j \in [m]\}$
\begin{eqnarray}
& {\sf BagSep}(B, B', \ell_2^2) = \frac{1}{mn}\sum_{i=1}^{i=n}\sum_{j=1}^{j=m}\|x_i-y_j\|_2^2 \nonumber \\
& = \frac{1}{n}\sum_{i=1}^{i=n}\|x_i\|_2^2 + \frac{1}{m}\sum_{j=1}^{j=m}\|y_j\|_2^2 - \frac{2}{mn}\sum_{i=1}^{i=n}\sum_{j=1}^{j=m} \langle x_i, y_j \rangle \nonumber \\
& = \frac{1}{n}\sum_{i=1}^{i=n}\|x_i\|_2^2 + \frac{1}{m}\sum_{j=1}^{j=m}\|y_j\|_2^2 - \frac{2}{mn} \langle \sum_{i=1}^{i=n}x_i, \sum_{j=1}^{j=m}y_j \rangle \nonumber
\end{eqnarray}
\end{proof}
\begin{lemma}
Given $\mc{B}$ and $B \in \mc{B}$,
\begin{eqnarray}
{\sf IntraBagSep}(B, \ell_2) & = & 2[\|B\| - \|\mu(B)\|_2^2] \nonumber \\
{\sf MeanInterBagSep}(\mc{B}, \ell_2) & = & \frac{2}{|\mc{B}|}\sum_{B \in \mc{B}}\|B\| \nonumber \\
& + & \frac{2}{|\mc{B}|(|\mc{B}|-1)}\left[\|\sum_{B \in \mc{B}}\mu(B)\|_2^2 - \sum_{B \in \mc{B}}\|\mu(B)\|_2^2\right] \nonumber
\end{eqnarray}
\end{lemma}
\begin{proof}
First part is trivial from Lemma \ref{lemma:ell_2_case}. If $\mc{B} = \{B_i | i \in [m]\}$,
\begin{eqnarray}
& & {\sf MeanInterBagSep}(\mc{B}, \ell_2) = \frac{1}{m(m-1)}\underset{i \not= j}{\sum_{i=1}^{m}\sum_{j=1}^{m}}{\sf BagSep}(B_i, B_j, \ell_2) \nonumber \\
& = & \frac{1}{m(m-1)}\underset{i \not= j}{\sum_{i=1}^{m}\sum_{j=1}^{m}}(\|B_i\| + \|B_j\| - \langle \mu(B_i), \mu(B_j) \rangle) \nonumber \\
& = & \frac{2}{m}\sum_{i=1}^{m}\|B_i\| \nonumber \\
& - & \frac{2}{m(m-1)}\left[\sum_{i=1}^{m}\sum_{j=1}^{m}\langle \mu(B_i), \mu(B_j) \rangle - \sum_{i=1}^{m}\|\mu(B_i)\|_2^2 \right] \nonumber \\
& = & \frac{2}{m}\sum_{i=1}^{m}\|B_i\| - \frac{2}{m(m-1)}\left[\|\sum_{i=1}^{m}\mu(B_i)\|_2^2 - \sum_{i=1}^{m}\|\mu(B_i)\|_2^2 \right] \nonumber
\end{eqnarray}
\end{proof}
Algorithm \ref{alg:two} is used to compute the Bag Separation for squared euclidean distance.\\
\begin{algorithm}
\caption{Compute Bag Separation with squared euclidean distance}\label{alg:two}
\KwData{Set of bags $\mathcal{B}$}
\KwResult{${\sf MeanIntraBagSep}(\mc{B}, \ell_2^2), {\sf MeanInterBagSep}(\mc{B}, \ell_2^2)$}
${\sf MeanIntraBagSep} \gets 0$\\
${\sf MeanInterBagSep} \gets 0$\\
${\sf AvgSqNorm} \gets [0]_{|\mathcal{B}|}$\\
${\sf BagMeans} \gets [0]_{|\mathcal{B}| x n}$\\
${\sf SumofAvgSqNorm} \gets 0$\\
${\sf SumofBagMeans} \gets [0]_{1 x n}$\\
${\sf SumofBagMeansNorms} \gets 0$\\
\For{$B \in \mathcal{B}$}{
    \For{$i \in B$}{
        ${\sf AvgSqNorm}(B) \gets {\sf AvgSqNorm}(B) + \|x^{(i)}\|_2^2$\\
        ${\sf BagMeans}(B) \gets {\sf BagMeans}(B) + x^{(i)}$\\
    }
    ${\sf AvgSqNorm}(B) \gets {\sf AvgSqNorm}(B)/|B|$\\
    ${\sf BagMeans}(B) \gets {\sf BagMeans}(B)/|B|$\\
}
\For{$B \in \mathcal{B}$}{
    ${\sf MeanIntraBagSep} \gets {\sf MeanIntraBagSep} + 2[{\sf AvgSqNorm}(B) - \|{\sf BagMeans}(B)\|_2^2]$\\
    ${\sf SumofAvgSqNorm} \gets {\sf SumofAvgSqNorm} + {\sf AvgSqNorm}(B)$\\
    ${\sf SumofBagMeans} \gets {\sf SumofBagMeans} + {\sf BagMeans(B)}$\\
    ${\sf SumofBagMeansNorms} \gets {\sf SumofBagMeansNorms} + \|{\sf BagMeans}(B)\|_2^2$\\
}
${\sf MeanIntraBagSep} \gets {\sf MeanIntraBagSep}/|\mc{B}|$\\
${\sf MeanInterBagSep} \gets \frac{2}{|\mc{B}|}{\sf SumofAvgSqNorms} - \frac{2}{|\mc{B}|(|\mc{B}|-1)}[\|{\sf SumofBagMeans}\|_2^2 - {\sf SumofBagMeansNorms}]$
\end{algorithm}
\FloatBarrier
\begin{theorem}
Assuming the Bags to be disjoint, the running time of Algorithm \ref{alg:two} is $O(mn + |\mathcal{B}|n + |\mathcal{B}|)$ where $m$ is the number of examples and n is the dimension of the input space.
\end{theorem}
\begin{proof}
Runtime = $\sum_{B \in \mathcal{B}}|B|n + \sum_{B \in \mathcal{B}}(1+n) = mn + |\mathcal{B}|n + |\mathcal{B}|$
\end{proof}
\subsection{Adversarial Example of Bags with Ratio of Mean Inter to Intra Bag Separation as 1/2}
\label{sec:example}
Consider $X = \{x^{(1)}, x^{(2)}, x^{(3)}\}$ which lie on a straight line. The distances are as follows:
\begin{itemize}
    \item $d(x^{(1)}, x^{(2)}) = d_1$
    \item $d(x^{(2)}, x^{(3)}) = d_2$
    \item $d(x^{(1)}, x^{(3)}) = d_1 + d_2$
\end{itemize}
We have two bags $B_1 = \{x^{(1)}, x^{(3)}\}$ and $B_2 = \{x^{(2)}\}$. The Intra-bag separations for both of them are as follows:
\begin{itemize}
    \item ${\sf BagSep}(B_1, B_1, d) = \frac{1}{2^2}(d(x^{(1)}, x^{(1)}) + d(x^{(1)}, x^{(3)}) + d(x^{(3)}, x^{(1)}) + d(x^{(3)}, x^{(3)})) = \frac{1}{2}(d_1 + d_2)$
    \item ${\sf BagSep}(B_2, B_2, d) = 0$
\end{itemize}
Hence, ${\sf MeanIntraBagSep}(\mc{B}, d) = \frac{1}{4}(d_1 + d_2)$. Now, the bag separation between the bags is as follows: 
\begin{itemize}
    \item ${\sf BagSep}(B_1, B_2, d) = \frac{1}{1 \times 2}(d(x^{(1)}, x^{(2)}) + d(x^{(3)}, x^{(2)})) = \frac{1}{2}(d_1 + d_2)$
    \item ${\sf InterBagSep}(B_1, d) = \frac{1}{2-1}({\sf BagSep}(B_1, B_2, d)) = \frac{1}{2}(d_1 + d_2)$
    \item ${\sf InterBagSep}(B_2, d) = \frac{1}{2-1}({\sf BagSep}(B_2, B_1, d)) = \frac{1}{2}(d_1 + d_2)$
\end{itemize}
Hence, ${\sf MeanInterBagSep}(\mc{B}, d) = \frac{1}{2}(d_1 + d_2)$.\\ Hence, ${\sf MeanInterBagSep}(\mc{B}, d)/{\sf MeanIntraBagSep}(\mc{B}, d) = 1/2$
\FloatBarrier

\section{Additional Details of Experimental Setup}\label{sec:additional-experiment-details}
We begin with additional details on the different baselines evaluated in our experiments.

${\bf DLLP}$: For a bag $B$, the  ${\sf DLLP\text{-}BCE}$ loss is given by ${\sf bce}(y_B/|B|, \hat{y}_B/|B|)$ where $y_B$ and $\hat{y}_B$ are the given and predicted label sums of bag $B$ where ${\sf bce}$ is the binary cross-entropy, that of ${\sf DLLP\text{-}MSE}$ is $(y_B - \hat{y}_B)^2$, and that of ${\sf DLLP\text{-}MAE}$ is $|y_B - \hat{y}_B|$. The minibatch loss is the sum of the per-bag losses over the $8$ bags in the minibatch. \\
${\bf GenBags}$: We divide the $8$ bags in a minibatch into $2$ blocks of $4$ bags each. For each block $60$ iid values of $\mb{w} = (w_1,\dots, w_4)$ are sampled from $N(\mb{0}, \bm{\Sigma})$ where $\bm{\Sigma}$ is the inner product matrix of the directions of the corners of the tetrahedron centered at origin. In particular, the diagonal entries of $\bm{\Sigma}$ are $1$ and all the off-diagonal entries are $-1/3$. This is a solution to the SDP of \cite{SRR} for the case of the $4$ bags in a minibatch being iid random, and we utilize this in our implementation. For each of the $60$ samples of $\mb{w}$ we create a generalized bag with those weights. In total we have $120$ generalized bags derived from a minibatch of $8$ bags.\\
${\bf Easy\text{-}LLP}$: We directly implement the soft-surrogate label loss given in Defn. 3.4 of  \cite{easy-llp} which is instantiated using the BCE loss at the instance-level. \\
${\bf OT}$ methods: There are pseudo-labeling techniques based on optimal transport proposed in \cite{OT}. We first train our model to convergence using ${\bf DLLP\text{-}BCE}$. Then we begin pseudo-labeling by constructing an OT problem described in Eqn 7 of \cite{OT}. We implement this without Entropic Regularization which we call ${\bf OT\text{-}LLP}$. We also implement it with Entropic Regularization. We call these ${\bf Hard\text{-}OT\text{-}LLP}$ and ${\bf Soft\text{-}OT\text{-}LLP}$ based on whether we use hard or soft pseudo-labels. \\
${\bf SIM\text{-}LLP}$: In this method proposed by \cite{KotziasDFS15}, the bag-level DLLP loss is augmented with a pairwise similarity based loss penalizing different predictions of geometrically close feature-vectors (Eqn 3 in \cite{KotziasDFS15}). Since the similarity based loss has number of terms which is square in the number of feature-vectors in a minibatch, we sample a random set of $400$ feature-vectors from each minbatch to apply this loss. \\
${\bf Mean\text{-}Map}$: The optimization given in Algorithm 1 of \cite{QuadriantoSCL09} is implemented in two steps. The quantities $\hat{\mu}_{XY}$ therein are first computed and then the computation for $\hat{\theta}^*$ is implemented using a minibatch optimization along the same lines as the above methods.

\section{Instance-level Model Training Results}
\label{appendix:instance_level_training}
\begin{flushleft}
We perform instance-level training of our model on Criteo CTR Dataset for comparison. We perform a train-test spilt of 80:20 on the dataset. We then train using instance level mini-batch gradient descent for the same number of epochs, using the same optimizer, model, learning rate schedule and the instance-level variant of the loss function. We obtain an AUC score of 80.1 with BCE loss and an AUC score of 79.94 with MSE loss.

We also perform instance-level training of our model on Criteo SSCL Dataset performing the same train-test spilt. We obtain an MSE of 147 after training with MSE loss.
\end{flushleft}

\section{Baseline Training Results}
\label{appendix:baseline_results}
\begin{flushleft}
Table \ref{tab:auc_baselines} reports the AUC scores of all the baselines on our bags created using Criteo CTR Dataset. We take the best test AUC score during each training configuration. The mean and standard deviation over 5 splits has been reported. Table \ref{tab:feat_bags_cssl_mse} reports the best MSE scores of all the baselines on our bags created using Criteo SSCL Dataset. Again, the mean and standard deviation over 5 splits has been reported.
\end{flushleft}

\section{Bag creation and filtering Statistics}
\label{appendix:bag_stats}
\begin{flushleft}
The statistics of datasets created as described in Section \ref{sec:dataset_creation} before and after clipping for all $349$ datasets are in Table \ref{tab:bag_stats}. We report the number of bags created, number of bags retained after clipping, percentage of instances left after clipping and the mean and standard deviation of bag size in each dataset. The datasets which are emboldened pass our filter and are used for training.
\end{flushleft}
\subsection{${\bf CumuBagSizeDist}$ for LLP-Bench datasets}
\label{appendix:bag_size_dist}
\begin{flushleft}
Table \ref{tab:bag_size_clusters} contains the threshold bags sizes such that $t\%$ of the bags have at most that size, for $t = 50,70,85,95$ for LLP-Bench datasets formed using Criteo CTR Dataset. The \textit{Tail size} cluster to which each dataset is assigned is also listed. Table \ref{tab:bag_size_percentile_criteo_sscl} contains the same for LLP-Bench datasets formed using Criteo SSCL Dataset.
\end{flushleft}
\subsection{${\bf  LabelPropStdev}$ and \textit{Label Variation} clusters for LLP-Bench datasets}
\label{appendix:std_label_prop_dist}
\begin{flushleft}
Table \ref{tab:label_prop_std} contains the ${\sf LabelPropStddev}$ for all LLP-Bench datasets formed using Criteo CTR Dataset. The \textit{Label Variation} cluster to which each dataset is assigned is also listed. Table \ref{tab:label_prop_std_SSCL} does the same for LLP-Bench datasets formed using Criteo SSCL Dataset.
\end{flushleft}
\subsection{${\sf InterIntraRatio}$ and \textit{Bag Separation} clusters for all datasets}
\label{appendix:bag_sep_stats}
\begin{flushleft}
Table \ref{tab:bag_sep_and_clusters} contains ${\sf MeanInterBagSep}$, ${\sf MeanIntraBagSep}$ and ${\sf InterIntraRatio}$ for all datasets in LLP-Bench formed using Criteo CTR Dataset. It also contains the information of the \textit{Bag Separation} cluster to which each of these datasets are assigned. Table \ref{tab:bagsep_criteo_sscl} does the same for LLP-Bench datasets formed using Criteo SSCL Dataset.
\end{flushleft}
\section{Training results on Fixed size Feature-bags Datasets}
\label{appendix:feature_random_train}
\begin{flushleft}
We also create and train our model on \textit{fixed size feature bags}. To create these datasets, we first perform a 5-fold split of the Criteo CTR Dataset. Next, for each group key $\mc{C}$ corresponding to an LLP-Bench dataset, we construct a random ordering of the train set with the constraint that feature vectors with same values of attributes in $\mc{C}$ lie in a contiguous segment. We then assign contiguous segments of size $k$ to the same bag to create \textit{fixed size feature bags} for $k \in \{64, 128, 256, 512\}$. We train our ${\sf DLLP}, {\sf GenBags}, {\sf Easy\text{-}LLP}, {\sf OT\text{-}LLP}, {\sf SIM\text{-}LLP}$ and ${\sf Mean\text{-}Map}$ baselines on these \textit{fixed size feature bags} and report the mean and std of test AUC scores in Table \ref{tab:feature_rand_auc}. For each group key $\mc{C}$, Table \ref{tab:feature_rand_auc} contains $4$ contiguous rows, one for bag size $\{64, 128, 256, 512\}$ each in ascending order.
\end{flushleft}
\section{Training results on Random Bags}
\label{appendix:random_bags}
\begin{flushleft}
Table \ref{tab:rand_feat_bags} is replicated along with the standard deviation of AUC scores across 5 splits in Table \ref{tab:rand_feat_bags_std}.
\end{flushleft}
\section{Feature Bag Datasets with Group Key Size $3$}\label{appendix:grp_key_size_3}
We also calculate the number of additional datasets which would have been created had we also considered group keys of size $3$. Using $instance_{thresh}\%$ as $30\%$, we retain $1195$ of ${26 \choose 3}$ datasets. It would be intractable to handle so many datasets and we believe that the current benchmark provides sufficient diversity.

\section{Analysis of label proportions of large bags}\label{appendix:large_bag_label_prop}
Large bags with label proportions close to $0$ or $1$ contain a considerable amount of information. In this section we calculate the percentage of instances which lie in large bags (bags with size greater than $high_{thresh} = 2500$) with skewed label proportions. We calculate this for each feature bag dataset. Since we throw these bags out of our dataset, if the percentage of such instances is low then we do not lose much information. We say that the label proportion of a bag is skewed if either it is less than $\eps$ or greater than $1 - \eps$. \\
Table \ref{tab:large_bag} reports these percentages for $\eps=0.1$ and $\eps=0.05$ respectively. It can be seen that the maximum percentages over all datasets with $\eps=0.1$ and $\eps=0.05$ are $6.66\%$ and $1.42\%$ respectively. This shows that there are relatively low number of such datapoints and they can be dropped so that neural network training is tractable.

\section{Dataset diversity analysis}\label{appendix:dataset_diversity}
Figure \ref{fig:all_three_scatter} shows a scatter plot with the three metrics, ${\sf MeanBagSize}$, ${\sf LabelPropStddev}$ and ${\sf InterIntraRatio}$ on its axes. Each point represents one of the datasets in our benchmark. Figures \ref{fig:meanbagsize_stddevlabelprop}, \ref{fig:stddevlabelprop_interintraratio} and \ref{fig:interintraratio_meanbagsize} shows the projections of Figure \ref{fig:all_three_scatter} on ${\sf MeanBagSize}$ vs ${\sf LabelPropStddev}$ plane, ${\sf LabelPropStddev}$ vs ${\sf InterIntraRatio}$ plane and ${\sf InterIntraRatio}$ vs ${\sf MeanBagSize}$ plane respectively. The diversity of LLP-Bench is more apparent from these scatter plots as diversity afforded due to combination of these metrics can be visualized.

For better readability and ease of reference we replicate (magnified versions of) Figures \ref{fig:datasets_vs_bag_metrics} and \ref{fig:datasets_performance_vs_bag_metrics} as Figures \ref{fig:datasets_vs_bag_metrics_bigger} and \ref{fig:datasets_performance_vs_bag_metrics_bigger} respectively.

\section{Cramer's V between grouping-key and label for LLP-Bench Datasets}\label{app:Cramers}
It is important to see how each grouping-key pair $(A, B)$ of LLP-Bench feature-bag datasets is correlated with the label. If the correlation is high, then most bags of the dataset corresponding to $(A, B)$ will have label proportions close to $0$ or $1$. On the other hand, bags have mixed labels if the correlation is low. Since both, the labels and grouping-key $(A, B)$ are categorical, we compute Cramer's V between them as follows. Given two categorical features $(X, Y)$ such that $X \in \{1, \dots, r\}$ and $Y \in \{1, \dots, c\}$, let $N$ be the total number of data points. Let $O_{i, j}$ be the total number of times $X$ takes the value $i$ and $Y$ takes the value $j$ for $(i, j) \in \{1, \dots, r\} \times \{1, \dots, c\}$. We define the expected occurence of this event assuming independence of $X$ and $Y$ as $E_{i, j} = Np_iq_j$ where $p_i = \sum_{j=1}^c O_{i, j}/N$ and $q_j = \sum_{i=1}^r O_{i, j}/N$. We define ${\sf Cramer's V}$ as follows.
\begin{align}
    {\sf Cramer's\, V}&  := \sqrt{(\chi^2/N)/\min(r-1, c-1)} \\ \nonumber
    & \text{where} \quad \chi^2 = \sum_{i=1}^r\sum_{j=1}^c\frac{(O_{i, j} - E_{i, j})^2}{E_{i, j}} \\ \nonumber
\end{align}
For our case, we use the pairs $(A, B)$ as $X$ and the label as $Y$. Thus, $r$ will be the number of bags in that dataset and $c = 2$ since it is a binary classification dataset. $O_{i, 0}$ and $O_{i, 1}$ will be the number of instances in the $i^{th}$ bag labeled $0$ and $1$ respectively. $N$ will be the total number of instances in the dataset. We report the $\chi^2$ values, total number of instances and ${\sf Cramer's \, V}$ for all LLP-Bench datasets in Table \ref{tab:cramers_v}. We observe a minimum ${\sf Cramer's\, V}$ of $0.22$ and a maximum ${\sf Cramer's\, V}$ of $0.39$. The ${\sf Cramer's\, V}$ is concentrated towards the maximium. We see sufficient diversity in LLP-Bench dataset with respect to ${\sf Cramer's\, V}$.

\section{Criteo SSCL Column Names}\label{app:criteo_sscl_column_encoding}
Table \ref{tab:criteo_sscl_column_encoding} has the names of the columns of the Criteo SSCL dataset.

\begin{figure}[htb]
    \centering
    \begin{subfigure}[b]{0.6\linewidth}
        \includegraphics[width=\linewidth]{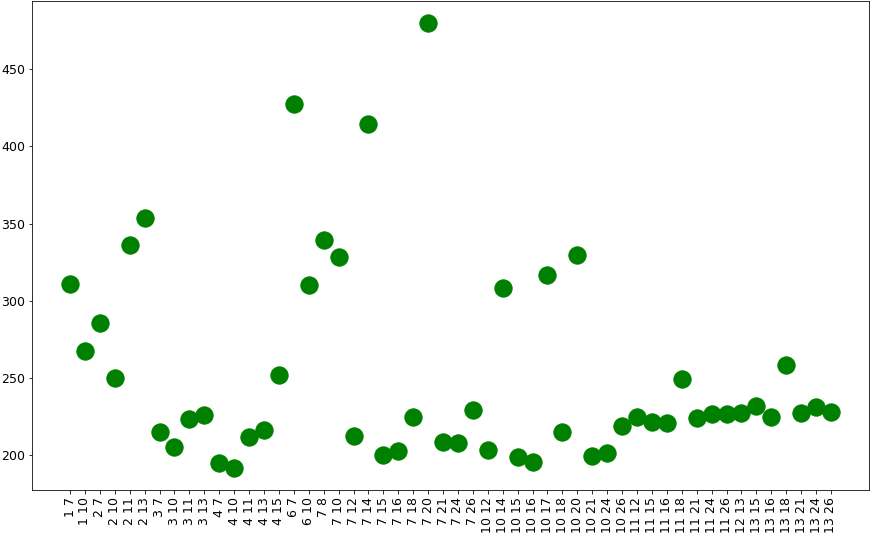}
        \caption{${\sf MeanBagSize}$}
        \label{fig:mean-bag-size_bigger}
    \end{subfigure}
    \begin{subfigure}[b]{0.6\linewidth}
        \includegraphics[width=\linewidth]{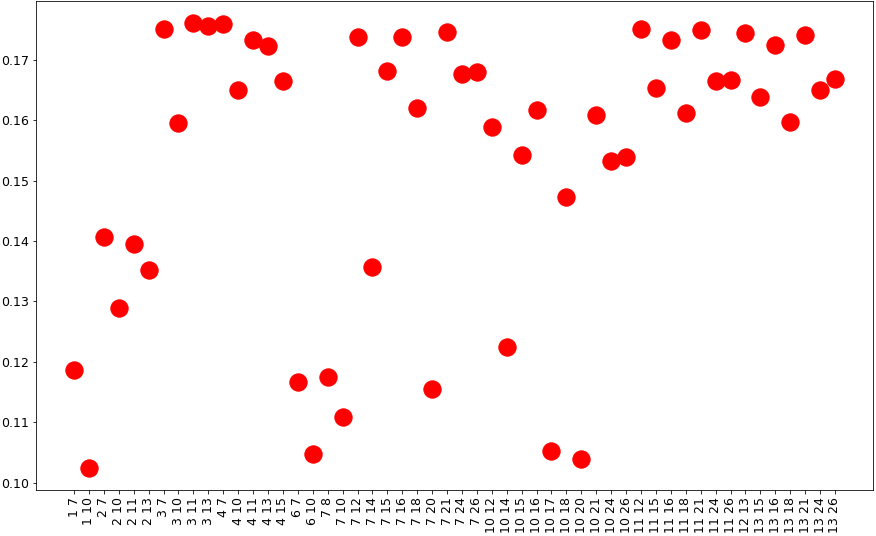}
        \caption{${\sf LabelPropStdev}$}
        \label{fig:label-prop-stdev_bigger}
    \end{subfigure}
    \begin{subfigure}[b]{0.6\linewidth}
        \includegraphics[width=\linewidth]{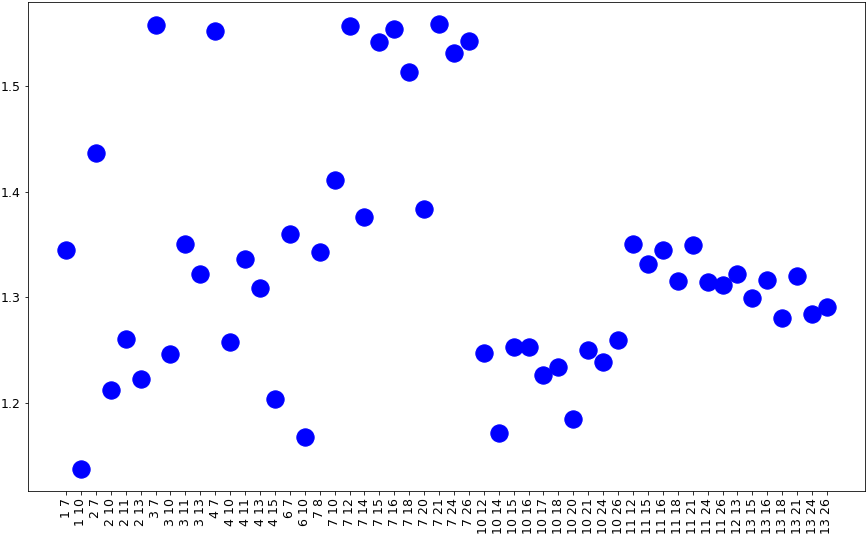}
        \caption{${\sf InterIntraRatio}$}
        \label{fig:inter-intra-mean_ratio_bigger}
    \end{subfigure}
    \caption{Datasets vs. bag-level metrics: $y$-axis has the metric, $x$-axis has the datsets. Replication of Figure \ref{fig:datasets_vs_bag_metrics}.
    }
    \label{fig:datasets_vs_bag_metrics_bigger}
\end{figure}
\begin{figure}[htb]
    \centering
    \begin{subfigure}[b]{0.6\linewidth}
        \includegraphics[width=\linewidth]{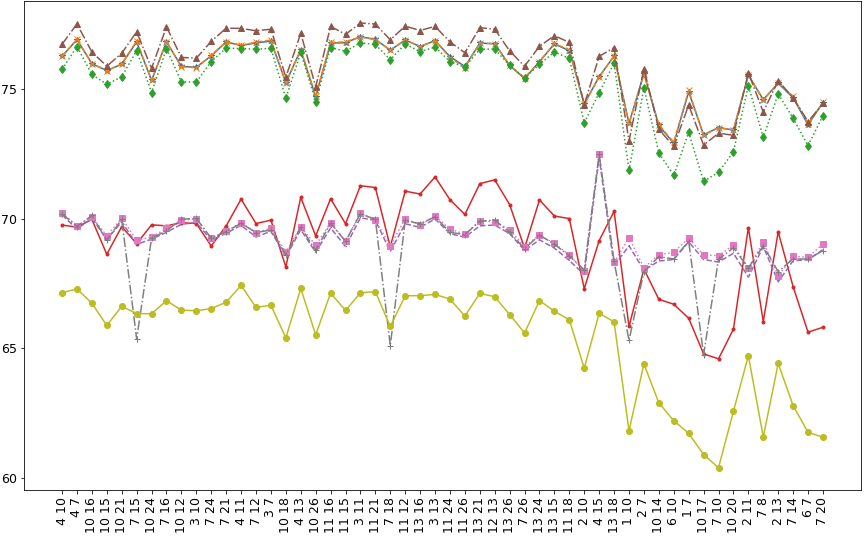}
        \caption{${\sf MeanBagSize}$}
        \label{fig:mean-bag-size-2_bigger}
    \end{subfigure}
    \begin{subfigure}[b]{0.6\linewidth}
        \includegraphics[width=\linewidth]{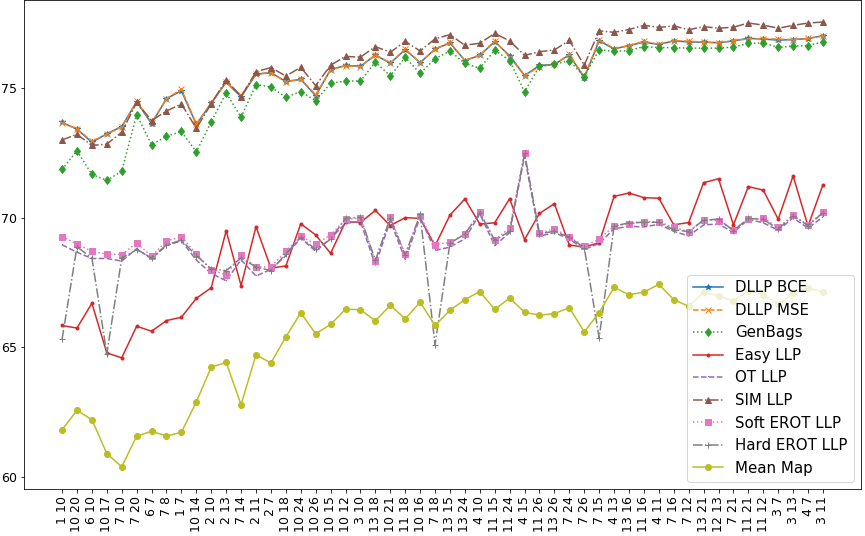}
        \caption{${\sf LabelPropStdev}$}
        \label{fig:label-prop-stdev-2_bigger}
    \end{subfigure}
    \begin{subfigure}[b]{0.6\linewidth}
        \includegraphics[width=\linewidth]{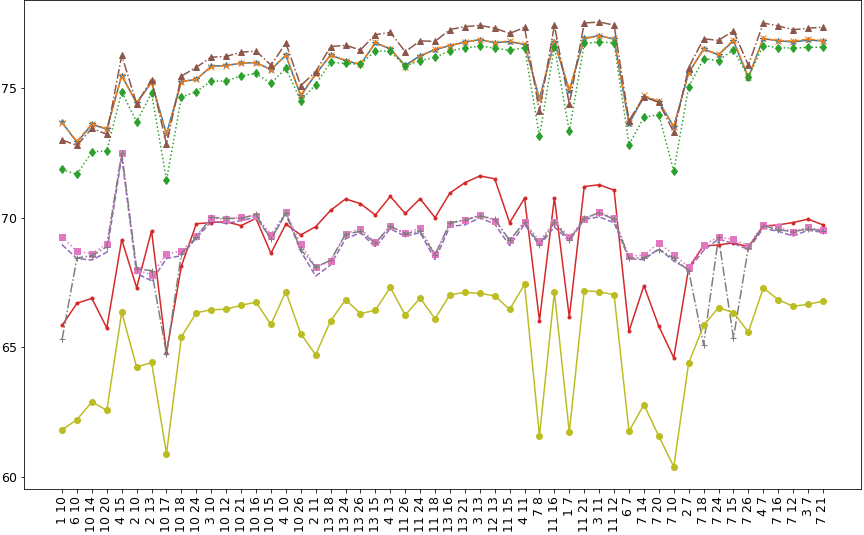}
        \caption{${\sf InterIntraRatio}$}
        \label{fig:inter-intra-mean_ratio-2_bigger}
    \end{subfigure}
    \caption{Datasets performance: AUC scores on the y-axis, $x$-axis has the datasets ordered according to increasing metric. Replication of Figure \ref{fig:datasets_performance_vs_bag_metrics}.
    }
    \label{fig:datasets_performance_vs_bag_metrics_bigger}
\end{figure}

\begin{table}[]
\scriptsize
\centering
\caption{Percentage of large bags with label proportion less than $\eps$ or greater than $1 - \eps$.}
\label{tab:large_bag}
\begin{tabular}{llcc}\toprule
\textbf{Col1} & \textbf{Col2} & \textbf{$\eps$ = 0.1} & \textbf{$\eps$ = 0.05} \\ \midrule
\textit{C1} & \textit{C7} & 4.18 & 0.64 \\
\textit{C1} & \textit{C10} & 0.56 & 0.09 \\
\textit{C2} & \textit{C7} & 4.05 & 1.16 \\
\textit{C2} & \textit{C10} & 3.19 & 0.13 \\
\textit{C2} & \textit{C11} & 3.7 & 1.11 \\
\textit{C2} & \textit{C13} & 4.15 & 1.06 \\
\textit{C3} & \textit{C7} & 2.12 & 0.51 \\
\textit{C3} & \textit{C10} & 1.88 & 0.34 \\
\textit{C3} & \textit{C11} & 2.23 & 0.63 \\
\textit{C3} & \textit{C13} & 2.3 & 0.57 \\
\textit{C4} & \textit{C7} & 1.45 & 0.62 \\
\textit{C4} & \textit{C10} & 2.35 & 0.3 \\
\textit{C4} & \textit{C11} & 1.63 & 0.61 \\
\textit{C4} & \textit{C13} & 2.0 & 0.61 \\
\textit{C4} & \textit{C15} & 6.66 & 1.42 \\
\textit{C6} & \textit{C7} & 5.77 & 0.85 \\
\textit{C6} & \textit{C10} & 0.79 & 0.1 \\
\textit{C7} & \textit{C8} & 4.83 & 0.74 \\
\textit{C7} & \textit{C10} & 6.28 & 1.06 \\
\textit{C7} & \textit{C12} & 1.96 & 0.5 \\
\textit{C7} & \textit{C14} & 6.23 & 1.07 \\
\textit{C7} & \textit{C15} & 1.51 & 0.72 \\
\textit{C7} & \textit{C16} & 1.5 & 0.5 \\
\textit{C7} & \textit{C18} & 2.54 & 0.92 \\
\textit{C7} & \textit{C20} & 5.64 & 1.07 \\
\textit{C7} & \textit{C21} & 1.78 & 0.51 \\
\textit{C7} & \textit{C24} & 2.01 & 0.79 \\
\textit{C7} & \textit{C26} & 3.49 & 0.72 \\
\textit{C10} & \textit{C12} & 1.92 & 0.34 \\
\textit{C10} & \textit{C14} & 1.31 & 0.44 \\
\textit{C10} & \textit{C15} & 3.87 & 0.79 \\
\textit{C10} & \textit{C16} & 1.96 & 0.33 \\
\textit{C10} & \textit{C17} & 3.45 & 0.13 \\
\textit{C10} & \textit{C18} & 3.64 & 0.53 \\
\textit{C10} & \textit{C20} & 0.7 & 0.11 \\
\textit{C10} & \textit{C21} & 1.91 & 0.33 \\
\textit{C10} & \textit{C24} & 2.63 & 0.3 \\
\textit{C10} & \textit{C26} & 2.22 & 0.26 \\
\textit{C11} & \textit{C12} & 2.08 & 0.59 \\
\textit{C11} & \textit{C15} & 1.82 & 0.7 \\
\textit{C11} & \textit{C16} & 1.69 & 0.57 \\
\textit{C11} & \textit{C18} & 2.49 & 0.92 \\
\textit{C11} & \textit{C21} & 1.84 & 0.6 \\
\textit{C11} & \textit{C24} & 2.16 & 0.76 \\
\textit{C11} & \textit{C26} & 3.35 & 0.59 \\
\textit{C12} & \textit{C13} & 2.21 & 0.57 \\
\textit{C13} & \textit{C15} & 2.28 & 0.7 \\
\textit{C13} & \textit{C16} & 1.99 & 0.61 \\
\textit{C13} & \textit{C18} & 3.07 & 0.95 \\
\textit{C13} & \textit{C21} & 2.11 & 0.57 \\
\textit{C13} & \textit{C24} & 2.46 & 0.72 \\
\textit{C13} & \textit{C26} & 3.28 & 0.51 \\ \bottomrule
\end{tabular}%
\end{table}

\FloatBarrier
\begin{figure}
     \centering
     \begin{subfigure}[b]{0.45\textwidth}
         \centering
         \includegraphics[width=\textwidth]{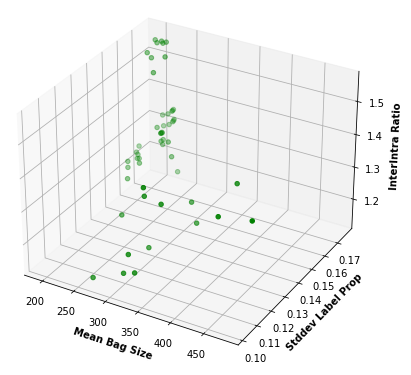}
         \caption{${\sf MeanBagSize}$ vs ${\sf LabelPropStddev}$ vs ${\sf InterIntraRatio}$}
         \label{fig:all_three_scatter}
     \end{subfigure}
     \hfill
     \begin{subfigure}[b]{0.45\textwidth}
         \centering
         \includegraphics[width=\textwidth]{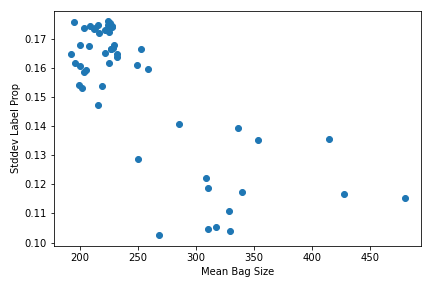}
         \caption{${\sf MeanBagSize}$ vs ${\sf LabelPropStddev}$}
         \label{fig:meanbagsize_stddevlabelprop}
     \end{subfigure}
     \hfill
     \begin{subfigure}[b]{0.45\textwidth}
         \centering
         \includegraphics[width=\textwidth]{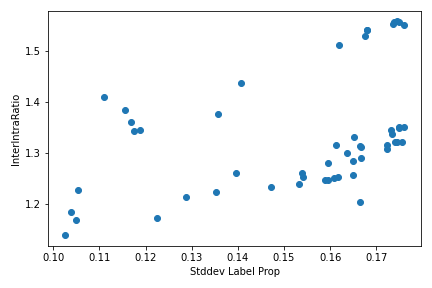}
         \caption{${\sf LabelPropStddev}$ vs ${\sf InterIntraRatio}$}
         \label{fig:stddevlabelprop_interintraratio}
     \end{subfigure}
     \begin{subfigure}[b]{0.45\textwidth}
         \centering
         \includegraphics[width=\textwidth]{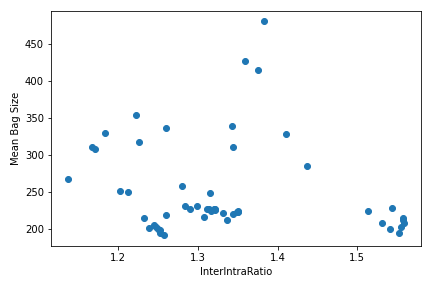}
         \caption{${\sf InterIntraRatio}$ vs ${\sf MeanBagSize}$}
         \label{fig:interintraratio_meanbagsize}
     \end{subfigure}
     \vspace{5pt}
        \caption{Scatter plots for ${\sf MeanBagSize}$ vs ${\sf LabelPropStddev}$ vs ${\sf InterIntraRatio}$ and it's projections for Criteo CTR Dataset}
        \label{fig:scatter_plots_metrics}
\end{figure}

\begin{figure}
     \centering
     \begin{subfigure}[b]{0.45\textwidth}
         \centering
         \includegraphics[width=\textwidth]{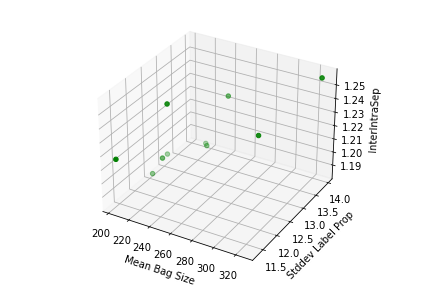}
         \caption{${\sf MeanBagSize}$ vs ${\sf LabelPropStddev}$ vs ${\sf InterIntraRatio}$}
         \label{fig:all_three_scatter_criteo_ssl}
     \end{subfigure}
     \hfill
     \begin{subfigure}[b]{0.45\textwidth}
         \centering
         \includegraphics[width=\textwidth]{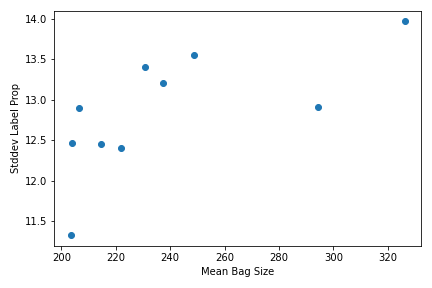}
         \caption{${\sf MeanBagSize}$ vs ${\sf LabelPropStddev}$}
         \label{fig:meanbagsize_stddevlabelprop_criteo_ssl}
     \end{subfigure}
     \hfill
     \begin{subfigure}[b]{0.45\textwidth}
         \centering
         \includegraphics[width=\textwidth]{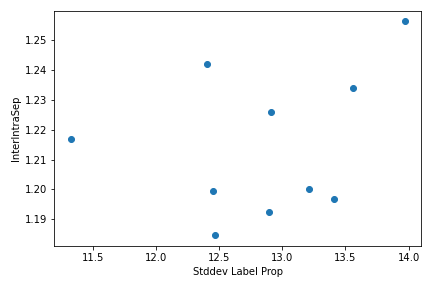}
         \caption{${\sf LabelPropStddev}$ vs ${\sf InterIntraRatio}$}
         \label{fig:stddevlabelprop_interintraratio_criteo_ssl}
     \end{subfigure}
     \begin{subfigure}[b]{0.45\textwidth}
         \centering
         \includegraphics[width=\textwidth]{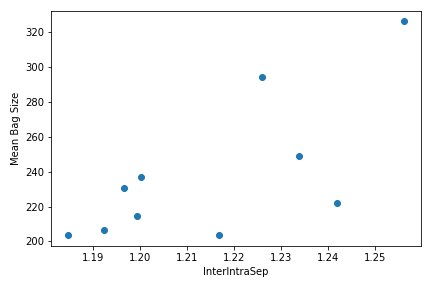}
         \caption{${\sf InterIntraRatio}$ vs ${\sf MeanBagSize}$}
         \label{fig:interintraratio_meanbagsize_criteo_ssl}
     \end{subfigure}
     \vspace{5pt}
        \caption{Scatter plots for ${\sf MeanBagSize}$ vs ${\sf LabelPropStddev}$ vs ${\sf InterIntraRatio}$ and it's projections for Criteo SSL Dataset}
        \label{fig:scatter_plots_metrics_criteo_ssl}
\end{figure}

\begin{table}[]
\scriptsize
\centering
\caption{Cramer's V between (Col1, Col2) and the label for each Dataset created with Criteo CTR.}
\label{tab:cramers_v}
\begin{tabular}{llccc} \toprule
\textbf{Col1} & \textbf{Col2} & \textbf{$\chi^2$} & \textbf{No. of inst.} & \textbf{Cramer's V} \\ \midrule
\textit{C1} & \textit{C7} & 1282596.95 & 17379031 & 0.27 \\
\textit{C1} & \textit{C10} & 736694.22 & 14870760 & 0.22 \\
\textit{C2} & \textit{C7} & 2347843.42 & 22344522 & 0.32 \\
\textit{C2} & \textit{C10} & 1511078.88 & 18901970 & 0.28 \\
\textit{C2} & \textit{C11} & 1786730.98 & 17917297 & 0.32 \\
\textit{C2} & \textit{C13} & 1504395.05 & 15984434 & 0.31 \\
\textit{C3} & \textit{C7} & 2503281.56 & 17363106 & 0.38 \\
\bottomrule
\end{tabular}
\end{table}

\onecolumn
\begin{table}[]
\scriptsize
\centering
\caption{Test AUC scores on training baselines on Datasets created using Criteo CTR Dataset.}
\label{tab:auc_baselines}
\begin{tabular}{llccccccccc}
\toprule
\textbf{Col1} & \textbf{Col2} & \textbf{${\sf DLLP\text{-}BCE}$} & \textbf{${\sf DLLP\text{-}MSE}$} & \textbf{${\sf GenBags}$} & \textbf{${\sf Easy\text{-}LLP}$} & \textbf{${\sf OT\text{-}LLP}$} & \textbf{${\sf SIM\text{-}LLP}$} & \textbf{${\sf Soft\text{-}EROT\text{-}LLP}$} & \textbf{${\sf Hard\text{-}EROT\text{-}LLP}$} & \textbf{${\sf Mean\text{-}Map}$} \\ \midrule
\textit{C1} & \textit{C7} & 74.91\tiny±0.06 & 74.97\tiny±0.04 & 73.34\tiny±0.08 & 66.16\tiny±0.38 & 69.11\tiny±0.07 & 74.4\tiny±0.08 & 69.26\tiny±0.08 & 69.14\tiny±0.06 & 61.71\tiny±0.01 \\
\textit{C1} & \textit{C10} & 73.69\tiny±0.05 & 73.67\tiny±0.07 & 71.87\tiny±0.1 & 65.84\tiny±0.09 & 68.96\tiny±0.08 & 73.01\tiny±0.05 & 69.27\tiny±0.05 & 65.3\tiny±8.55 & 61.81\tiny±0.04 \\
\textit{C2} & \textit{C7} & 75.61\tiny±0.03 & 75.61\tiny±0.04 & 75.06\tiny±0.03 & 68.06\tiny±0.19 & 68.01\tiny±0.09 & 75.81\tiny±0.03 & 68.11\tiny±0.07 & 67.96\tiny±0.09 & 64.39\tiny±0.03 \\
\textit{C2} & \textit{C10} & 74.44\tiny±0.04 & 74.42\tiny±0.04 & 73.71\tiny±0.05 & 67.3\tiny±0.15 & 67.85\tiny±0.2 & 74.41\tiny±0.02 & 67.99\tiny±0.07 & 68.03\tiny±0.12 & 64.23\tiny±0.03 \\
\textit{C2} & \textit{C11} & 75.55\tiny±0.04 & 75.57\tiny±0.03 & 75.15\tiny±0.05 & 69.66\tiny±0.14 & 67.75\tiny±0.16 & 75.65\tiny±0.03 & 68.1\tiny±0.12 & 68.1\tiny±0.1 & 64.71\tiny±0.02 \\
\textit{C2} & \textit{C13} & 75.24\tiny±0.06 & 75.24\tiny±0.04 & 74.83\tiny±0.06 & 69.5\tiny±0.05 & 67.57\tiny±0.17 & 75.34\tiny±0.04 & 67.8\tiny±0.13 & 67.96\tiny±0.13 & 64.41\tiny±0.02 \\
\textit{C3} & \textit{C7} & 76.87\tiny±0.04 & 76.91\tiny±0.04 & 76.59\tiny±0.04 & 69.95\tiny±0.34 & 69.52\tiny±0.11 & 77.33\tiny±0.03 & 69.63\tiny±0.12 & 69.57\tiny±0.05 & 66.66\tiny±0.03 \\
\bottomrule
\end{tabular}
\end{table}
\begin{table}[]
\scriptsize
\centering
\caption{Bag size at 50, 70, 85, 95 \%ile with tail size clusters for Datasets created with Criteo CTR.}
\label{tab:bag_size_clusters}
\begin{tabular}{llccccc}
\toprule
\textbf{Col1} & \textbf{Col2} & \textbf{Bag size: 50 \%ile} & \textbf{Bag size: 70 \%ile} & \textbf{Bag size: 85 \%ile} & \textbf{Bag size: 95 \%ile} & \textbf{Tail size cluster} \\ \midrule
\textit{C1} & \textit{C7} & 140 & 271 & 569 & 1270 & Long-tailed \\
\textit{C1} & \textit{C10} & 122 & 224 & 454 & 1068 & Short-tailed \\
\textit{C2} & \textit{C7} & 136 & 253 & 503 & 1102 & Short-tailed \\
\textit{C2} & \textit{C10} & 119 & 210 & 420 & 957 & Short-tailed \\
\textit{C2} & \textit{C11} & 157 & 311 & 637 & 1329 & Long-tailed \\
\textit{C2} & \textit{C13} & 165 & 334 & 688 & 1398 & Long-tailed \\
\textit{C3} & \textit{C7} & 103 & 173 & 334 & 804 & Very Short-tailed \\
\bottomrule
\end{tabular}
\end{table}
\begin{table}[]
\scriptsize
\caption{Label Variation cluster and LabelPropStdDev values for Datasets created using Criteo CTR.}
\label{tab:label_prop_std}
\centering
\begin{tabular}{llcc}
\toprule
\textbf{Col1} & \textbf{Col2} & \textbf{${\sf LabelPropStdev}$} & \textbf{Label Variation Cluster} \\ \midrule
\textit{C1} & \textit{C7} & 0.12 & Very Low \\
\textit{C1} & \textit{C10} & 0.1 & Very Low \\
\textit{C2} & \textit{C7} & 0.14 & Low \\
\textit{C2} & \textit{C10} & 0.13 & Low \\
\textit{C2} & \textit{C11} & 0.14 & Low \\
\textit{C2} & \textit{C13} & 0.14 & Low \\
\textit{C3} & \textit{C7} & 0.18 & High \\
\bottomrule
\end{tabular}%

\end{table}
\begin{table}[]
\scriptsize
\centering
\caption{BagSep statistics along with the Bag Separation Cluster for each dataset created using Criteo CTR.}
\label{tab:bag_sep_and_clusters}
\begin{tabular}{llcccc} \toprule
\textbf{Col1} & \textbf{Col2} & \textbf{MeanInterBagSep} & \textbf{MeanIntraBagSep} & \textbf{InterIntraRatio} & \textbf{Bag Separation cluster} \\ \midrule
\textit{C1} & \textit{C7} & 0.81 & 0.6 & 1.34 & Well-separated \\
\textit{C1} & \textit{C10} & 0.83 & 0.73 & 1.14 & Less-separated \\
\textit{C2} & \textit{C7} & 0.78 & 0.55 & 1.44 & Well-separated \\
\textit{C2} & \textit{C10} & 0.81 & 0.67 & 1.21 & Less-separated \\
\textit{C2} & \textit{C11} & 0.76 & 0.6 & 1.26 & Medium-separated \\
\textit{C2} & \textit{C13} & 0.75 & 0.61 & 1.22 & Medium-separated \\
\textit{C3} & \textit{C7} & 0.84 & 0.54 & 1.56 & Far-separated \\
\bottomrule
\end{tabular}
\end{table}
\begin{small}
\begin{centering}
\begin{longtable}{llccccccc}
\caption{Test AUC scores of training baselines on fixed size feature bags dataset using Criteo CTR. Each grouping-key $({\sf Col1}, {\sf Col2})$ corresponds to $4$ contiguous rows, one for bag size $\{64, 128, 256, 512\}$ each in ascending order.}
\label{tab:feature_rand_auc}
\\ \toprule
\textbf{Col1} & \textbf{Col2} & \textbf{${\sf DLLP\text{-}BCE}$} & \textbf{${\sf DLLP\text{-}MSE}$} & \textbf{${\sf GenBags}$} & \textbf{${\sf Easy\text{-}LLP}$} & \textbf{${\sf OT\text{-}LLP}$} & \textbf{${\sf SIM\text{-}LLP}$} & \textbf{${\sf Mean\text{-}Map}$} \\ \midrule
\textit{C1} & \textit{C7} & 77.6\tiny±0.03 & 77.63\tiny±0.02 & 77.31\tiny±0.03 & 64.01\tiny±0.05 & 72.32\tiny±0.03 & 76.54\tiny±0.1 & 62.2\tiny±0.03 \\
\textit{C1} & \textit{C7} & 76.99\tiny±0.02 & 77.07\tiny±0.03 & 76.7\tiny±0.02 & 64.52\tiny±0.07 & 69.98\tiny±0.06 & 75.2\tiny±0.15 & 61.88\tiny±0.03 \\
\textit{C1} & \textit{C7} & 76.27\tiny±0.02 & 76.42\tiny±0.03 & 75.91\tiny±0.03 & 64.82\tiny±0.07 & 68.41\tiny±0.16 & 74.24\tiny±0.15 & 61.67\tiny±0.03 \\
\textit{C1} & \textit{C7} & 75.51\tiny±0.02 & 75.76\tiny±0.02 & 75.09\tiny±0.06 & 65.77\tiny±0.2 & 67.88\tiny±0.08 & 73.51\tiny±0.1 & 61.49\tiny±0.02 \\
\textit{C1} & \textit{C10} & 77.56\tiny±0.02 & 77.57\tiny±0.01 & 77.23\tiny±0.04 & 60.39\tiny±0.04 & 73.3\tiny±0.07 & 76.36\tiny±0.13 & 61.59\tiny±0.0 \\
\textit{C1} & \textit{C10} & 76.9\tiny±0.03 & 76.96\tiny±0.01 & 76.59\tiny±0.02 & 61.02\tiny±0.13 & 70.83\tiny±0.1 & 74.56\tiny±0.15 & 61.07\tiny±0.01 \\
\bottomrule
\end{longtable}%
\end{centering}
\end{small}

\begin{table}[]
\caption{AUC scores of training baselines of Random Bags created from Criteo CTR dataset with error bars.}
\label{tab:rand_feat_bags_std}
\centering
\begin{tabular}{lllll}
\toprule
Bag size & \textbf{64} & \textbf{128} & \textbf{256} & \textbf{512} \\ \midrule
${\sf DLLP\text{-}BCE}$ & 77.54\tiny±0.02 & 76.96\tiny±0.02 & 76.24\tiny±0.04 & 75.22\tiny±0.05 \\
${\sf DLLP\text{-}MSE}$ & 77.56\tiny±0.02 & 77.03\tiny±0.02 & 76.33\tiny±0.03 & 75.42\tiny±0.05 \\
${\sf GenBags}$ & 77.08\tiny±0.03 & 76.5\tiny±0.03 & 75.75\tiny±0.08 & 75.22\tiny±0.05 \\
${\sf Easy\text{-}LLP}$ & 75.69\tiny±0.18 & 74.18\tiny±0.54 & 72.32\tiny±0.11 & 70.13\tiny±0.47 \\
${\sf OT\text{-}LLP}$ & 74.25\tiny±0.1 & 71.53\tiny±0.09 & 68.1\tiny±0.07 & 65.26\tiny±0.48 \\
${\sf SIM\text{-}LLP}$ & 77.41\tiny±0.03 & 76.73\tiny±0.08 & 75.47\tiny±0.16 & 73.13\tiny±0.27 \\
${\sf Soft\text{-}EROT\text{-}LLP}$ & 74.43\tiny±0.07 & 71.82\tiny±0.09 & 68.34\tiny±0.16 & 65.16\tiny±0.51 \\
${\sf Hard\text{-}EROT\text{-}LLP}$ & 74.23\tiny±0.06 & 71.65\tiny±0.1 & 68.08\tiny±0.15 & 65.54\tiny±0.4 \\
${\sf Mean\text{-}Map}$ & 63.17\tiny±0.03 & 63.06\tiny±0.05 & 62.83\tiny±0.06 & 62.26\tiny±0.06 \\ \bottomrule
\end{tabular}%

\end{table}
\begin{scriptsize}
\begin{centering}
\begin{longtable}{llccccc}
\caption{Bag Level Statistics of all the Groupings formed using Criteo CTR Dataset (Emboldened : Used for Training)}
\label{tab:bag_stats}
\\ \toprule
\textbf{Col1} & \textbf{Col2} & \textbf{No. bags created} & \textbf{No. bags after clipping} & \textbf{$\%$ Inst. after clipping} & \textbf{Mean bag size} & \textbf{Stddev of bag sizes} \\  \midrule
\textit{C1} & \textit{-} & 1443 & 1261 & 0.7 & 256.16 & 386.61 \\ 
\textit{C3} & \textit{-} & 175781 & 39052 & 18.42 & 216.17 & 316.89 \\ 
\textit{C4} & \textit{-} & 128509 & 38802 & 21.68 & 256.11 & 359.9 \\ 
\textit{C7} & \textit{-} & 11930 & 7839 & 12.39 & 724.69 & 636.25 \\ 
\textit{C8} & \textit{-} & 629 & 531 & 0.34 & 295.24 & 446.88 \\ 
\textit{C10} & \textit{-} & 41224 & 20252 & 17.01 & 384.92 & 482.64 \\
\bottomrule
\end{longtable}

\end{centering}
\end{scriptsize}

\begin{scriptsize}
\begin{centering}
\begin{longtable}{llccccc}
\caption{Bag Level Statistics of all the Groupings with Criteo SSCL (Emboldened : Used for Training)}
\label{tab:dataset_filtering_criteo_sscl} \\
\toprule
\textbf{Col1} & \textbf{Col2} & \textbf{No. Bags created} & \textbf{No. Bags after clipping} & \textbf{\% Inst. after clipping} & \textbf{Mean bag size} & \textbf{Stddev of Label Prop.} \\
\midrule
$\tilde{C}$1 & - & 11 & 3 & 0.12 & 672.33 & 5.01 \\
$\tilde{C}$2 & - & 5 & 0 & 0.0 &  &  \\
$\tilde{C}$3 & - & 6402 & 1627 & 21.29 & 226.74 & 11.71 \\
$\tilde{C}$4 & - & 15 & 7 & 0.38 & 939.43 & 8.28 \\
$\tilde{C}$5 & - & 7723 & 2095 & 28.84 & 238.55 & 13.29 \\
$\tilde{C}$6 & - & 22 & 2 & 0.15 & 1277.0 & 9.22 \\
\bottomrule
\end{longtable}%
\end{centering}
\end{scriptsize}
\begin{table}[htb]
\small
\caption{Bag size at 50, 70, 85, 95 \%ile for Criteo SSCL Dataset}
\label{tab:bag_size_percentile_criteo_sscl}
\centering
\begin{tabular}{llccccc}
\toprule
\textbf{Col1} & \textbf{Col2} & \textbf{Bag size: 50 \%ile} & \textbf{Bag size: 70 \%ile} & \textbf{Bag size: 85 \%ile} & \textbf{Bag size: 95\%ile} & \textbf{Tail-ssize cluster} \\
\midrule
$\tilde{C}$2 & $\tilde{C}$5 & 106 & 174 & 303 & 727 & Short-tailed \\
$\tilde{C}$3 & $\tilde{C}$5 & 102 & 178 & 355 & 833 & Short-tailed\\
$\tilde{C}$3 & $\tilde{C}$8 & 100 & 165 & 304 & 778 & Short-tailed\\
$\tilde{C}$3 & $\tilde{C}$16 & 127 & 250 & 530 & 1217 & Long-tailed\\
$\tilde{C}$5 & $\tilde{C}$6 & 117 & 198 & 378 & 827 & Medium-tailed\\
$\tilde{C}$5 & $\tilde{C}$7 & 115 & 194 & 387 & 927 & Medium-tailed\\
$\tilde{C}$5 & $\tilde{C}$8 & 108 & 180 & 347 & 750 & Short-tailed\\
$\tilde{C}$5 & $\tilde{C}$9 & 108 & 182 & 336 & 684 & Short-tailed\\
$\tilde{C}$5 & $\tilde{C}$16 & 120 & 212 & 418 & 907 & Medium-tailed\\
$\tilde{C}$8 & $\tilde{C}$16 & 147 & 286 & 625 & 1372 & Long-tailed\\
\bottomrule
\end{tabular}%
\end{table}

\begin{table}[]
\caption{BagSep statistics for each dataset formed with Criteo SSCL Dataset}
\label{tab:bagsep_criteo_sscl}
\centering
\begin{tabular}{llcccc}
\toprule
Col1 & Col2 & ${\sf MeanInterBagSep}$ & ${\sf MeanIntraBagSep}$ & ${\sf InterIntraRatio}$ & \textbf{Bag Separation cluster}\\
\midrule
$\tilde{C}$2 & $\tilde{C}$5 & 39174.79 & 33071.37 & 1.18 &Less-separated\\
$\tilde{C}$3 & $\tilde{C}$5 & 38570.65 & 31055.51 & 1.24 &Medium-separated\\
$\tilde{C}$3 & $\tilde{C}$8 & 40055.06 & 32919.02 & 1.22 &Medium-separated\\
$\tilde{C}$3 & $\tilde{C}$16 & 39716.97 & 32398.83 & 1.23 &Medium-separated\\
$\tilde{C}$5 & $\tilde{C}$6 & 39344.84 & 32880.52 & 1.2 & Less-separated\\
$\tilde{C}$5 & $\tilde{C}$7 & 39543.59 & 32951.49 & 1.2 &Less-separated\\
$\tilde{C}$5 & $\tilde{C}$8 & 39540.86 & 32970.05 & 1.2 &Less-separated\\
$\tilde{C}$5 & $\tilde{C}$9 & 39384.92 & 33032.5 & 1.19 &Less-separated\\
$\tilde{C}$5 & $\tilde{C}$16 & 39132.17 & 31713.31 & 1.23 &Medium-separated\\
$\tilde{C}$8 & $\tilde{C}$16 & 38848.47 & 30922.17 & 1.26 &Well-separated\\
\bottomrule
\end{tabular}%
\end{table}

\begin{table}[]
\caption{Label Variation cluster and LabelPropStdDev values for Datasets created using Criteo SSCL.}
\label{tab:label_prop_std_SSCL}
\centering
\begin{tabular}{llcc}
\toprule
\textbf{Col1} & \textbf{Col2} & \textbf{${\sf LabelPropStdev}$} & \textbf{Label Variation Cluster} \\ \midrule
$\tilde{C}$2 & $\tilde{C}$5 & 12.47 & Medium \\
$\tilde{C}$3 & $\tilde{C}$5 & 12.4 & Medium \\
$\tilde{C}$3 & $\tilde{C}$8 & 11.33 & Low \\
$\tilde{C}$3 & $\tilde{C}$16 & 12.91 & Medium \\
$\tilde{C}$5 & $\tilde{C}$6 & 13.41 & High \\
$\tilde{C}$5 & $\tilde{C}$7 & 13.21 & High \\
$\tilde{C}$5 & $\tilde{C}$8 & 12.45 & Medium \\
$\tilde{C}$5 & $\tilde{C}$9 & 12.9 & Medium \\
$\tilde{C}$5 & $\tilde{C}$16 & 13.56 & High \\
$\tilde{C}$8 & $\tilde{C}$16 & 13.97 & High \\
\bottomrule
\end{tabular}%
\end{table}
\begin{table}[]
\centering
\caption{Test AUC scores on training baselines on Bag Datasets created using Criteo SSCL dataset}
\label{tab:feat_bags_cssl_mse}
\begin{tabular}{llcccc}
\toprule
\multicolumn{1}{c}{Col1} & \multicolumn{1}{c}{Col2} & ${\sf DLLP}-{\sf MAE}$ & ${\sf DLLP}-{\sf MSE}$ & ${\sf GenBags}$ & ${\sf SIM}-{\sf LLP}$ \\
\midrule
$\tilde{C}$2 & $\tilde{C}$5 & 183.58$\pm$\tiny{4.18} & 181.84$\pm$\tiny{4.0} & 178.84$\pm$\tiny{2.55} & 182.19$\pm$\tiny{1.14} \\
$\tilde{C}$3 & $\tilde{C}$5 & 192.53$\pm$\tiny{2.02} & 193.5$\pm$\tiny{3.16} & 194.02$\pm$\tiny{3.18} & 192.49$\pm$\tiny{1.35} \\
$\tilde{C}$3 & $\tilde{C}$8 & 169.3$\pm$\tiny{0.6} & 168.31$\pm$\tiny{1.84} & 169.49$\pm$\tiny{2.75} & 167.6$\pm$\tiny{1.05} \\
$\tilde{C}$3 & $\tilde{C}$16 & 190.48$\pm$\tiny{2.42} & 187.36$\pm$\tiny{1.27} & 188.07$\pm$\tiny{2.52} & 186.56$\pm$\tiny{2.82} \\
$\tilde{C}$5 & $\tilde{C}$6 & 198.6$\pm$\tiny{1.53} & 193.96$\pm$\tiny{3.47} & 202.39$\pm$\tiny{2.98} & 194.82$\pm$\tiny{2.41} \\
$\tilde{C}$5 & $\tilde{C}$7 & 191.07$\pm$\tiny{1.33} & 186.72$\pm$\tiny{1.22} & 190.26$\pm$\tiny{1.67} & 186.29$\pm$\tiny{2.04} \\
$\tilde{C}$5 & $\tilde{C}$8 & 189.83$\pm$\tiny{2.57} & 185.28$\pm$\tiny{2.28} & 186.47$\pm$\tiny{1.74} & 184.68$\pm$\tiny{2.63} \\
$\tilde{C}$5 & $\tilde{C}$9 & 183.95$\pm$\tiny{1.48} & 182.04$\pm$\tiny{1.08} & 182.2$\pm$\tiny{1.84} & 181.42$\pm$\tiny{2.14} \\
$\tilde{C}$5 & $\tilde{C}$16 & 183.35$\pm$\tiny{1.62} & 180.48$\pm$\tiny{1.0} & 181.95$\pm$\tiny{2.09} & 180.36$\pm$\tiny{2.06} \\
$\tilde{C}$8 & $\tilde{C}$16 & 192.06$\pm$\tiny{2.46} & 190.82$\pm$\tiny{1.41} & 199.26$\pm$\tiny{1.84} & 191.1$\pm$\tiny{3.41}  \\
\bottomrule
\end{tabular}
\end{table}
\begin{table}[]
\caption{Encoding of columns for the Criteo SSCL Dataset}
\label{tab:criteo_sscl_column_encoding}
\centering
\begin{tabular}{cc}
\toprule
Column Name & Column Code \\
\midrule
sales\_amount\_in\_euro & Y \\
time\_delay\_for\_conversion & $\tilde{N}$1 \\
nb\_clicks\_1week & $\tilde{N}$2 \\
product\_price & $\tilde{N}$3 \\
product\_age\_group & $\tilde{C}$1 \\
device\_type & $\tilde{C}$2 \\
audience\_id & $\tilde{C}$3 \\
product\_gender & $\tilde{C}$4 \\
product\_brand & $\tilde{C}$5 \\
product\_category\_1 & $\tilde{C}$6 \\
product\_category\_2 & $\tilde{C}$7 \\
product\_category\_3 & $\tilde{C}$8 \\
product\_category\_4 & $\tilde{C}$9 \\
product\_category\_5 & $\tilde{C}$10 \\
product\_category\_6 & $\tilde{C}$11 \\
product\_category\_7 & $\tilde{C}$12 \\
product\_country & $\tilde{C}$13 \\
product\_id & $\tilde{C}$14 \\
product\_title & $\tilde{C}$15 \\
partner\_id & $\tilde{C}$16 \\
user\_id & $\tilde{C}$17 \\
\bottomrule
\end{tabular}%

\end{table}

\end{document}